\documentclass[11pt]{article}
\usepackage{amsmath,amssymb,graphicx,latexsym,soul,color,hyperref,setspace,mathrsfs,framed,pdflscape,xcolor,booktabs}
\hypersetup{colorlinks=true,linkcolor=blue,citecolor=blue,filecolor=blue,urlcolor=blue}

\def\tr{{\raise0pt\hbox{$\scriptscriptstyle\top$}}}

\renewcommand{\le}{\leqslant}
\renewcommand{\ge}{\geqslant}

\newtheorem{proposition}{Proposition}
\newtheorem{theorem}{Theorem}
\newtheorem{corollary}{Corollary}

\newtheorem{remark}{Remark}
\newtheorem{assume}{Assumption}

\newcommand{\Real}{\mathbb{R}}

\def\eop{\hfill{$\Box$}\medskip}

\newenvironment{proof}{\textbf{Proof}}{\eop}

\title{\vspace{-1.0 cm}\bf When fractional quasi p-norms concentrate} 
	
\author{
Ivan Y. Tyukin\footnote{Department of Mathematics, King's College London (e-mail: ivan.tyukin@kcl.ac.uk}
\and
Bogdan Grechuk\footnote{School of Computing and Mathematical Sciences, University of Leicester  (e-mail: bg83@leicester.ac.uk)}
\and
Evgeny Mirkes\footnote{School of Computing and Mathematical Sciences, University of Leicester (e-mail: em322@leicester.ac.uk)}
\and
Alexander N. Gorban\footnote{School of Computing and Mathematical Sciences, University of Leicester (e-mail: a.n.gorban@leicester.ac.uk); King's College London}
}

\date{}

\begin{document}

\maketitle

\begin{abstract}

 Concentration of distances in high dimension is an important factor for the development and design of stable and reliable data analysis algorithms. In this paper, we address the fundamental long-standing question about the concentration of distances in high dimension for fractional quasi $p$-norms, $p\in(0,1)$. The topic has been at the centre of various theoretical and empirical controversies. Here we, for the first time, identify conditions when fractional quasi $p$-norms concentrate and when they don't.
We show that contrary to some earlier suggestions,  for broad classes of  distributions, fractional quasi $p$-norms admit exponential and uniform in $p$ concentration bounds. For these distributions, the results effectively rule out previously proposed approaches to alleviate concentration by "optimal" setting the values of $p$ in $(0,1)$. At the same time, we specify conditions and the corresponding families of distributions for which one can still control concentration rates by appropriate choices of $p$. We also show that in an arbitrarily small vicinity of a distribution from a large class of distributions for which uniform concentration occurs, there are uncountably many other distributions featuring anti-concentration properties.  Importantly, this behavior enables devising relevant data encoding or representation schemes favouring or discouraging distance concentration. The results shed new light on this long-standing problem and resolve the tension around the topic in both theory and empirical evidence reported in the literature.   

\end{abstract}



\section{Introduction}

Distance-based methods are ubiquitous in modern data science, machine learning, and data-driven artificial intelligence. They are popular in bio-informatics \cite{zou2024common}, they constitute a core of various clustering \cite{oyewole2023data} and classification algorithms \cite{syriopoulos2023k}, are being routinely used in
other data analysis methods such as topological data analysis \cite{wasserman2018topological}, \cite{ezugwu2022comprehensive}, and in software packages for dimensionality
reduction \cite{glielmo2022dadapy}. As data collection capabilities grow over time, the data we collect often comprise vectors or tensors with a large number of attributes. Examples include images, health records, digitalized sound, texts, or DNA sequences. Due to the large number of attributes in such data, they are often called high-dimensional, albeit the notion of dimension of data is substantially more nuanced than that of the linear vector space the data belong to (see e.g. \cite{bac2021scikit}, \cite{sutton2023relative}, \cite{CAMASTRA20032945} and references therein).


However, when the number of attributes per data vector is large, the classical distance-based methods employing  $\ell^p$, $p\geq 1$ distances may become unstable \cite{pestov2013k}. This is because in high dimension, for a broad class of distributions, $\ell^p$ distances concentrate near their expected values \cite{beyer1999nearest}, \cite{biau2015high}, \cite{fleury2010concentration}. Hence the task of assigning a data point to a class or a cluster merely on the basis of $\ell^p$ distances could be sensitive to small computational perturbations and inaccuracies inherent in the data.  The issue is particularly severe in the {\it post-classical} limit, as coined by David Donoho in his lecture at the turn of the millennium \cite{donoho2000high,GorbanGMST:entropy}. In this limit, {\color{black} the number of data samples}  would typically be of the same order or even smaller than {\color{black} the number of attributes in a data vector}. Traditional methods of dimensionality reduction based on the assessment of the spectrum of covariance matrices may not work in this limit, thanks to the well-known effects such as the Marchenko-Pastur law \cite{marchenko-pastur}. Alternatives, such as compressed sensing \cite{donoho2006compressed}, have recently been brought into the spotlight \cite{bastounis2017absence}, \cite{bastounis2021extended}, \cite{shchukina2017pitfalls}, including in the context of potential instabilities and computability of sparsified features.

As an alternative to efforts addressing the curse of dimensionality at its core via dimensionality reduction, an interesting approach was hypothesized in \cite{aggarwal2001surprising} whereby controlling concentration could be achieved by changing the way we {\it measure} distances. In particular,  it has been proposed in \cite{aggarwal2001surprising} that concentration effects could be alleviated, or at the very least tamed, if one replaces the classical well-behaved metric with a fractional quasi-metric. In particular, with $\ell^p$ quasi-norms, where $p\in(0,1)$. {\color{black} These quasi-norms obey standard norm axioms except for the triangle inequality (sub-additivity)}. This work sparked remarkable interest in the community and led to extensive empirical studies \cite{mirkes2020fractional}, \cite{flexer2015choosing}, indicating that sometimes switching to quasi-norms helps and sometimes does not. In Section \ref{sec:discussion} we reproduced some of these experiments (please see Table \ref{tab:DBs}). 
Unfortunately, being empirical, these results do not provide a definite answer to the question of the relevance of quasi- and fractional- norms for avoiding or managing unwanted concentration in high dimension. The question, therefore, remained open to date.

In this work, we provide a rigorous answer to the long-standing controversy surrounding fractional quasi- $p$-norms in the context of their potential ability to cope with the curse of dimensionality. In particular, we address the  following open question: 

\vspace{2mm}

 {\it If and when employing fractional quasi-norms could {\it qualitatively} help to overcome the $\ell^p$, $p\in(0,\infty]$ concentration?}

\vspace{2mm}



Below we discuss the origin of the problem in more technical details. We then proceed with an overview of our contribution, including main results and the structure of the work. 


\subsection{Motivation for fractional $\ell^p$, $p\in(0,1)$ quasi-norms}
Several authors (see, e.g. \cite{beyer1999nearest}, \cite{biau2015high}) showed that under some mild assumptions, the standard $\ell^p$ distance, $p\geq 1$, of a random $n$-dimensional vector ${\bf x}=(x_1,x_2,\dots,x_n)$ concentrates around its expected value as $n\rightarrow \infty$. In particular, if $x_i$ are random variables with i.i.d. components satisfying $|x_i|<C$, and with $\mu_p=E[|x_i|^p]$, then for any $\delta>0$, $p\geq 1$, and sufficiently large $n\in\mathbb{N}$ the following estimate holds true (Corollary 2, \cite{biau2015high}):
\[
\mathbb{P}\left(\left|\frac{\|{\bf x}\|_p}{E[\|{\bf x}\|_p]}-1\right|\geq \delta \right)\leq 2 \exp\left(-\frac{\delta^2 n^{2/p - 1} \mu_p^{2/p}}{2C^2}+o
\left(n^{2/p-1}\right)\right).
\]

In a highly cited paper \cite{aggarwal2001surprising}, Aggarwal et al. showed that for a pair of ${\bf x}_1, {\bf x}_2$ drawn independently from the equidistribution in $[0,1]^n$ the following asymptotic property holds for all $p>0$
\begin{equation}\label{eq:aggarwal:eq}
\lim_{n\rightarrow\infty} E\left[\frac{|\|{\bf x}_1\|_p-\|{\bf x}_2\|_p|}{n^{1/p-1/2}}\right]=G \frac{1}{(p+1)^{1/p}} \frac{1}{(2p+1)^{1/2}},
\end{equation}
where $G$ is a positive constant which does not depend on $p$. The explicit dependence of the right-hand side of (\ref{eq:aggarwal:eq}) on $p$ motivated a suggestion that sometimes it may be possible to mitigate the curse of dimensionality by choosing a sufficiently small value of $p>0$. Moreover, a straightforward application of the Hoeffding inequality to $n$-dimensional random vectors ${\bf x}$ in the example above results in:
\[
\mathbb{P}\left(\frac{\|{\bf x}\|_p}{(n\mu_p)^{1/p}}-1>\delta\right)\leq \exp\left(-\left(\frac{((1+\delta)^p-1)^2\mu_p^2}{2C^2}\right)n\right),
\]
creating space for the hope that one may be able to control the concentration rates by the values of $p$.  However, empirical evidence presented in a systematic numerical study \cite{mirkes2020fractional} across different benchmark datasets did not support such a hope leaving the original question of the possibility to leverage the rate of $\ell^p$  quasi-norms concentrations in high-dimensional spaces open.

\subsection{Contribution and structure of this work}

In this paper, we show that, on the one hand, and contrary to the expectation presented in \cite{aggarwal2001surprising}, there is a fundamental limitation to attempts to alleviate the impact of quasi- $p$-norms concentration by varying the value of $p$ in $\|{\bf x}\|_p$. This limitation applies to a wide class of distributions, including  the ones considered in \cite{aggarwal2001surprising}. 
 On the other hand, there exist distributions for which $\ell^p$ quasi-norms concentrations may be mitigated by a choice of $p$. 
In particular,

\vspace{2mm}

\begin{itemize}
\item[(i)] we prove that, for a broad class of distributions (which includes those considered in \cite{aggarwal2001surprising}) from which vectors ${\bf x}$ are sampled, the following bound holds:
\[
\mathbb{P}\left(1-\delta \leq \frac{\|{\bf x}\|_p}{(n\mu_p)^{1/p}} \leq 1+\delta\right)\geq 1-2 \exp [-n\cdot f^\ast(\delta)]
\]
where $f^{\ast}(\delta)>0$ is separated away from zero for all $\delta>0$;
\item[(ii)] at the same time, we identify a class $\mathcal{A}$ of distributions for which $\ell^p$ quasi-norm concentration around expected values may be controlled by an appropriate choice of the value of $p$; this corresponds to anti-concentration;

\item[(iii)] moreover, we show that, in an arbitrary small neighbourhood of any product distribution with bounded support there is a distribution from the class $\mathcal{A}$; in other words, anti-concentration can be enforced by arbitrarily small perturbations of data. 
\end{itemize}

These results settle the controversy around $\ell^p$ quasi-norm concentration in a rather unexpected way. Indeed, in agreement with earlier empirical studies \cite{mirkes2020fractional} and in an apparent contradiction to \cite{aggarwal2001surprising}, which we shall explain later, we show that $\ell^p$ quasi-norms inevitably concentrate exponentially in $n$ and uniformly in $p$ for a large class of data distributions, including for the ones analyzed in \cite{aggarwal2001surprising}. However, as a refinement to the conclusions of \cite{mirkes2020fractional}, and in agreement with the suggestion from \cite{aggarwal2001surprising} (but for a totally different class of distributions), we demonstrate that there exist distributions for which the concentration rate could be effectively controlled by varying the values of $p$. Moreover, for a large class of distributions for which uniform concentration does occur, there are ``nearby'' distributions exhibiting anti-concentration. 

\vspace{2mm}

{\color{black}
In addition to settling the controversy, theoretical results presented in the paper bear direct  implications for several core areas in data science and machine learning. These areas include data preprocessing practices such as imputation and encoding and dimensionality reduction. In Supplementary Materials we illustrate these applications with relevant examples in detail. 
}

Technical content of the work is as follows. In Section \ref{sec:general} we specify the the main technical assumptions on the classes of distributions considered in this work. These are followed by the formal statement of our main results (Theorem \ref{theorem:p-norm-concentration}). Exponential rates in these  general probability bounds are dependent on $p$. Exploring this dependence analytically, we identify a class of distributions for which these rates converge to $0$ as $p\rightarrow 0$ (Section \ref{sec:non_uniform}). Following this observation, we formulate  non-concentration or anti-concentration results (Proposition \ref{thm:concentration_breaks}, Theorem \ref{thm:concentration_breaks:general}) confirming that one may indeed plausibly control or mitigate $\ell^p$ quasi-norm concentration (and thus technically address the curse of dimensionality) for data sampled from the family of identified distributions by setting the values of $p$ sufficiently small.  In Sections \ref{sec:small_p} -- \ref{sec:small_delta} we show that under an additional mild assumption, that the probability of $x_i$ falling near $0$ is not too large,  these exponential rates are bounded from below by functions that are independent of $p$. This settles the question asked in \cite{aggarwal2001surprising} but the settlement is opposite to what was initially suggested in \cite{aggarwal2001surprising}. We start with considering two special cases: when $p$ is small, including $p\rightarrow 0$, (Section \ref{sec:small_p}) and when $p$ is large, including $p\rightarrow\infty$ (Section \ref{sec:large_p}). Taking these limiting cases into account, we formulate a statement showing the existence of exponential concentration bound that is valid for all $p\in(0,\infty]$ (Theorem \ref{prop:unip}). A guide to main theorems is shown in Table \ref{tab:summary}.
\begin{table}
\centering
\begin{tabular}{|c|c|c|c|}
\hline
\hline
 & All admissible $p$ & Small $p$ & Large $p$ \\
\hline
\hline
No concentration & $-$ & Proposition \ref{thm:concentration_breaks} & $-$\\
& & Theorem \ref{thm:concentration_breaks:general} &  \\
& & Corollary \ref{cor:anti} & \\
\hline
Exponential & & & \\
(uniform) & Theorem \ref{prop:unip} & Proposition \ref{prop:pto0} & Proposition \ref{prop:infty}\\
concentration & & & \\
\hline 
Exponential  &  &  & \\
(non-uniform) & Theorem \ref{theorem:p-norm-concentration} &  Theorem \ref{theorem:p-norm-concentration} & Theorem \ref{theorem:p-norm-concentration}\\
concentration & & & \\
\hline
\end{tabular}
\vspace{1mm}
\caption{A brief guide on main theoretical results. Theorem \ref{theorem:p-norm-concentration} presents general exponential, albeit, non-uniform bounds on the probability of the concentration of $\ell_p$ quasi-norms. Following the analysis of its  assumptions, we present Proposition \ref{thm:concentration_breaks} and Theorem  \ref{thm:concentration_breaks:general} stating the existence of distributions for which concentration of $\ell_p$ quasi-norms can be controlled by an appropriate choice of $p$. Theorem \ref{prop:unip} and Propositions \ref{prop:pto0}, \ref{prop:infty} specify conditions when concentration is both exponential and uniform in $p$.}\label{tab:summary}
\end{table}
Relevant special cases, including for small deviations from expected values of $\ell^p$ quasi-norms are then presented in Section \ref{sec:small_delta} (see also Table \ref{tab:special}). In Section \ref{sec:discussion} we discuss theoretical and practical implications of our results.  {\color{black} Additional examples illustrating the theory on real and synthetic datasets are provided in Supplementary Materials.} Section \ref{sec:conclusion} concludes the paper. 

\begin{table}
\centering
\begin{tabular}{|c|c|}
\hline
\hline
Distribution & Result \\
\hline
\hline
Uniform  &  Proposition \ref{prop:cube}, Proposition \ref{prop:cubemon}, Remark \ref{rem:special_cases}\\
\hline
Difference &  Proposition \ref{prop:cube_diff}, Remark \ref{rem:special_cases} \\
of Uniform & \\
\hline
Standard & Remark \ref{rem:special_cases}\\
Normal & \\
\hline
\end{tabular}
\vspace{1mm}
\caption{Special cases and bounds for common distributions.}\label{tab:special}
\end{table}

\section{Main assumptions and the general estimate}\label{sec:general}

\subsection{Exponential concentration of fractional and quasi $p$-norms}
Let ${\bf x}=(x_1,\dots,x_n)$ be a random vector in ${\mathbb R}^n$. We will study the concentration properties of p-norm $||{\bf x}||_p := \left(\sum_{i=1}^n |x_i|^p\right)^{1/p}$ of ${\bf x}$, and how these properties depend on parameter $p>0$. We will make several assumptions on the distribution of ${\bf x}$. Our first assumption is:
\begin{assume}\label{assume:a} The components $x_i$ of ${\bf x}$ are independent and identically distributed (i.i.d.) random variables with 
{\color{black} $\mathbb{P}(|x_i|=c)\neq 1$ for any constant $c \in {\mathbb R}$.}
\end{assume}

Assumption \ref{assume:a} means that ${\bf x}$ follows the product distribution $\nu^n$ with the component distribution $\nu$ not being fully concentrated at $0$. The last requirement rules out the pathological case where all attributes are zero, in which case the $\ell_p$ lengths of all data vectors are $0$ for any $p>0$.  To avoid the situation where some $p$-norm may have infinite expectation, we assume that the distribution $\nu$ of components $x_i$ is either bounded (that is, has bounded support) or decays exponentially fast. More formally, we assume that
\begin{assume}\label{assume:b}
There exist positive constants $p_0>0$ and $t_0>0$ such that
\[
E[e^{t_0|x_i|^{p_0}}]<\infty.
\]
\end{assume}
Assumption \ref{assume:b} holds with every $p_0$ if the distribution of $x_i$ is bounded, with $p_0=2$ for normal distribution, with $p_0=1$ for exponential distribution, but fails if $x_i$ has a power-like tail. {\color{black} It is clear that the value of $p_0$ may depend on the distribution $\nu$ of $x_i$, and could therefore be written as $p_0(v)$. In what follows, in this and similar cases where such dependencies are clear from the context, we will omit these dependencies in the text to simplify the notation.}

Because $e^{t_0|x_i|^{p_0}}$ grows faster than $|x_i|^p$ for any $p$, Assumption \ref{assume:b} implies that all moments 
$$
\mu_p:=E[|x_i|^p], \quad p\geq 0,
$$ 
exist and finite. It also implies that $E[e^{t|x_i|^{p_0}}]<\infty$ for all $t\in (-\infty, t_0]$. Therefore, since for any $t>0$ and $p\in(0,p_0)$ there is an $M(t,p)>0$ such that $t_0|x_i|^{p_0}\geq t|x_i|^{p}$ for all $|x_i|\geq M(t,p)$, Assumption \ref{assume:b} implies that 
\begin{equation}\label{eq:bimplies}
E[e^{t|x_i|^p}]<\infty, \quad \forall p\in(0,p_0),\,\, \text{ and }\,\, \forall \ t \in {\mathbb R}. 
\end{equation}
Also, observe that Assumption \ref{assume:a}  and Assumption \ref{assume:b} imply that $\mu_p>0$ for all $p\geq 0$.

For large $n$, the central limit theorem 
implies that the sample mean $\frac{1}{n}\sum_{i=1}^n |x_i|^p$ converges in distribution to the normal distribution having mean $\mu_p$, 
hence $||{\bf x}||_p \approx (n\mu_p)^{1/p}$. The next theorem states that the ratio $\frac{||{\bf x}||_p}{(n\mu_p)^{1/p}}$ is concentrated around $1$ with probability {\it exponentially} close to $1$ as $n$ grows. 

\begin{theorem}\label{theorem:p-norm-concentration} Suppose that ${\bf x}$ follows the product distribution $\nu^n$ satisfying Assumptions \ref{assume:a} and \ref{assume:b}. Then, for any $\delta\in(0,1)${\color{black}, any $p\in(0,p_0]$, and any $n\geq 1$}

\begin{equation}\label{eq:genchern}
\begin{split}
&{\mathbb P}\left(\frac{||{\bf x}||_p}{(n\mu_p)^{1/p}} \geq 1+\delta \right) \leq \exp[-n \cdot \Lambda^+(p,\delta,\nu)], 
\\ 
& \\
&{\mathbb P}\left(\frac{||{\bf x}||_p}{(n\mu_p)^{1/p}} \leq 1-\delta \right) \leq \exp[-n \cdot \Lambda^-(p,\delta,\nu)], 
\end{split}
\end{equation}
where {\color{black}$\Lambda^+(p,\delta,\nu)$ and $\Lambda^-(p,\delta,\nu)>0$ are defined by}: 
\begin{equation}\label{eq:lambdadef}
\begin{split}
& \Lambda^\pm(p,\delta,\nu) := \sup\limits_{t \geq 0} \lambda^\pm(t,p,\delta,\nu), \\
\text{where} \quad
& \lambda^\pm(t,p,\delta,\nu) := \pm t(1 \pm \delta)^p \mu_p - \log(E[e^{\pm t|x_i|^p}]).
\end{split}
\end{equation}
{\color{black} The real-valued functions $\Lambda^+(p,\delta,\nu), \Lambda^-(p,\delta,\nu)$ are independent from $n$ and are the} best possible in the sense that
\begin{equation}\label{eq:genchern:optimal}
\begin{split}
\lim_{n\rightarrow \infty}& \ \frac{1}{n}\log {\mathbb P}\left(\frac{||{\bf x}||_p}{(n\mu_p)^{1/p}} \geq 1+\delta \right) = - \Lambda^+(p,\delta,\nu), 
\\ 
& \\
\lim_{n\rightarrow\infty}&\ \frac{1}{n}\log {\mathbb P}\left(\frac{||{\bf x}||_p}{(n\mu_p)^{1/p}} \leq 1-\delta \right) = - \Lambda^-(p,\delta,\nu). 
\end{split}
\end{equation}
\end{theorem}
\begin{proof}
We have
$$
{\mathbb P}\left(\frac{||{\bf x}||_p}{(n\mu_p)^{1/p}} \geq 1+\delta \right) = {\mathbb P}\left(\sum_{i=1}^n |x_i|^p \geq (1+\delta)^p \mu_p \cdot n \right),
$$
and similarly,
$$
{\mathbb P}\left(\frac{||{\bf x}||_p}{(n\mu_p)^{1/p}} \leq 1-\delta \right) = {\mathbb P}\left(\sum_{i=1}^n (-|x_i|^p) \geq -(1-\delta)^p \mu_p \cdot n \right).
$$
Chernoff's inequality \cite[p. 21]{Boucheron2013} states that, for any random variable $S$, and any real number $u$,
$$
{\mathbb P}[S \geq u] \leq \exp[-\psi^*_S(u)],
$$
where 
$$
\psi^*_S(u) := \sup\limits_{t \geq 0} (t u - \log(E[e^{t S}])),
$$
provided that $E[e^{t S}]$ is finite on some interval.
If $S=\sum_{i=1}^n z_i$ for i.i.d. random variables $z_i$, then $\psi^*_S(u)=n\psi^*_Z\left(\frac{u}{n}\right)$ \cite[p. 24]{Boucheron2013}. Assumption \ref{assume:b} implies that we can apply the Chernoff's inequality with $z_i=\pm|x_i|^p$ and $u=\pm(1\pm\delta)^p \mu_p \cdot n$, and \eqref{eq:genchern} follows.
The optimality of $\Lambda^\pm(p,\delta)$ in the exponents in \eqref{eq:genchern} follows from Cram\'er's theorem \cite[Theorem 2.1]{Pham}.
{\color{black} It states that if $(z_i)_{i\ge 1}$ are i.i.d.\ random variables and $\psi(t):=\log\mathbb{E}[e^{tz_1}]$ is finite in a neighbourhood of $0$, then for any $a>\mathbb{E}[z_1]$,
\[
\lim_{n\to\infty}\frac{1}{n}\log\mathbb{P}\!\left(\frac{1}{n}\sum_{i=1}^n z_i \ge a\right)
= - \sup_{t\ge 0}\bigl\{t a-\psi(t)\bigr\}.
\]
Applying this with $z_i=|x_i|^p$ and $a=(1+\delta)^p\mu_p$ gives the first limit in \eqref{eq:genchern:optimal}. The second limit in
\eqref{eq:genchern:optimal} follows by applying the same statement to $z_i=-|x_i|^p$ and $a=-(1-\delta)^p\mu_p$.}


Inequalities \eqref{eq:genchern} guarantee concentration of norm $||{\bf x}||_p$ in the interval 
\[
\left((1-\delta)(n\mu_p)^{1/p}, (1+\delta)(n\mu_p)^{1/p}\right)
\] with exponentially high probability, provided that constants $\Lambda^\pm(p,\delta,\nu)$ are strictly positive. We will prove that this is indeed the case. To do this we will need the following auxiliary result.

\begin{proposition}\label{prop:convergence}
Let $Z$ be a random variable such that $E[|Z|]<+\infty$ and $E[e^{t_0 Z}]<+\infty$ for some $t_0>0$. Then
\begin{equation}\label{eq:convergence}
\lim\limits_{t\to 0+}E\left[\frac{e^{t Z}-1}{t}\right]=E[Z].
\end{equation}
\end{proposition}
\begin{proof}
We will start with the following claim: for any $a\in\Real$, $h(t)=\frac{e^{at}-1}{t}$ is a non-decreasing function on $(0,\infty)$, with $\lim\limits_{t\to 0+}h(t)=a$. 

Indeed, $h'(t)=g(t)/t^2$, where $g(t)=ae^{at}t-(e^{at}-1)$. Then $g'(t)=a^2e^{at}t+ae^{at}-ae^{at}=a^2e^{at}t \geq 0$, hence $g(t)$ is non-decreasing on $(0,\infty)$. Thus $g(t)\geq g(0)=0$, and $h'(t)=g(t)/t^2 \geq 0$ for $t\in(0,\infty)$. 

The claim implies that $a \leq h(t) \leq h(t_0)$ for all $t\in(0,t_0]$. Hence, $|h(t)| \leq \max\{|a|,|h(t_0)|\}$.
Applying the claim to $a=Z(\omega)$ for each element $\omega$ of the underlying probability space, we conclude that $\lim\limits_{t\to 0+}\frac{e^{t Z}-1}{t}=Z$ pointwise, and that $\left|\frac{e^{t Z}-1}{t}\right|\leq W := \max\left\{|Z|,\left|\frac{e^{t_0 Z}-1}{t_0}\right|\right\}$. Because $E|Z|<\infty$, and $E\left[\left|\frac{e^{t_0 Z}-1}{t_0}\right|\right]\leq \frac{1}{t_0}(E[e^{t_0 Z}]+1)<\infty$, we have 
{\color{black} $\mathbb{E}[W]<\infty$.
Define $Y_t := (e^{tZ}-1)/t$. Then $Y_t \to Z$ pointwise as $t\to 0+$, and $|Y_t|\le W$ for all $t\in(0,t_0]$, where $\mathbb{E}[W]<\infty$. Hence, by the dominated convergence theorem,
$\lim_{t\to 0+}\mathbb{E}[Y_t]=\mathbb{E}[Z]$, which is exactly \eqref{eq:convergence}.}
\end{proof}

We can now proceed with the proof of the strict positivity of $\Lambda^\pm(p,\delta,\nu)$. 
Assumption \ref{assume:b} guarantees that the conditions of Proposition \ref{prop:convergence} hold true for $Z=\pm |x_i|^p$, and \eqref{eq:convergence} implies that 
$$
\lim\limits_{t\to 0+}E\left[\frac{e^{\pm t|x_i|^p}-1}{t}\right]=E[\pm|x_i|^p]=\pm \mu_p.
$$ 
Denoting $I(t):=E[e^{\pm t|x_i|^p}]$, this implies that
$$
\lim\limits_{t\to 0+}\frac{\log I(t)}{t}=\lim\limits_{I(t)\to 1}\frac{\log I(t)}{I(t)-1}\lim\limits_{t\to 0+}\frac{I(t)-1}{t} = \pm \mu_p,
$$
and
\begin{equation}\label{eq:tto0}
\lim\limits_{t\to 0+}\frac{\lambda^\pm(t,p,\delta,\nu)}{t}=\lim\limits_{t\to 0+}\frac{\pm \mu_p((1 \pm \delta)^p \cdot t -\log I(t)}{t}=\pm \mu_p((1 \pm \delta)^p -1) > 0.
\end{equation}
for all $p\in(0,p_0]$ and all $\delta\in(0,1)$. 

Because $\lambda^\pm(0,p,\delta,\nu)=0$, \eqref{eq:tto0} implies that there exists a $t'>0$ such that $\lambda^\pm(t',p,\delta,\nu)>0$, and hence
$\Lambda^\pm(p,\delta,\nu) = \sup\limits_{t \geq 0} \lambda^\pm(t,p,\delta,\nu)>0$.

This completes the proof of the theorem.
\end{proof}

Theorem \ref{theorem:p-norm-concentration} guarantees the concentration of norm $||{\bf x}||_p$ with high probability if $p\in (0,p_0)$ and $\delta$ are fixed and $n$ increases. However, this does not rule out the possibility of mitigating the concentration of $\ell^p$ quasi-norms if $p$ is allowed to vary with $n$. More precisely, the following question remains:    

\vspace{2mm}

{\it Can the concentration be avoided or minimized by an appropriate choice of $p=p(n,\delta,\nu)$?}

\vspace{2mm}

This possibility would be plausible for data distributions that satisfy
\begin{equation}\label{eq:concentration_necessary}
\lim\limits_{p \to 0+}\Lambda^+(p,\delta,\nu)=0 \ \mathrm{or} \ \lim\limits_{p \to 0+}\Lambda^-(p,\delta,\nu)=0.
\end{equation} 
As we show below, there is a class of distributions that satisfies properties (\ref{eq:concentration_necessary}). Furthermore, despite {\color{black} the fact that} Theorem \ref{theorem:p-norm-concentration} provides only upper bounds on the probabilities of the events ${||{\bf x}||_p}/{(n\mu_p)^{1/p}} \geq 1+\delta$ and ${||{\bf x}||_p}/{(n\mu_p)^{1/p}} \leq 1-\delta$, 
we will prove (see Theorem \ref{thm:concentration_breaks:general} below) that for distributions from this class there is indeed a possibility to alleviate the $\ell^p$ quasi-norm concentration by selecting $p>0$ sufficiently small for each given dimension $n$.

\subsection{Non-uniformity of exponential concentration  and mitigating quasi-norm concentration by the choice of $p$}\label{sec:non_uniform} Here we demonstrate that Assumptions \ref{assume:a} and \ref{assume:b} are not sufficient to warrant exponential $\ell^p$ quasi-norm concentration. To do so, we will present an example of a distribution that satisfies Assumptions \ref{assume:a} and \ref{assume:b} and yet for which conditions (\ref{eq:concentration_necessary}) hold. Moreover, we will also show that distributions from this class may effectively escape the curse of dimensionality via appropriate choices of $p$. 

Let $x_i$ be a random variable supported on $\{0,r\}$, $r>0$, and satisfying
\begin{equation}\label{eq:example_concentration_breaks}
\mathbb{P}(x_i=0)=a, \ \mathbb{P}(x_i=r)=1-a, \ a\in(0,1). 
\end{equation}
Let $\nu_0=\nu_0(a,r)$ be the corresponding probability measure. 
In this case, $\mu_p=r^p(1-a)$, and $E[e^{\pm t|x_i|^p}]=a+e^{\pm t r^p}(1-a)$. Therefore
\[
\lambda^{\pm}(t,p,\delta,\nu_0)=\pm t(1\pm\delta)^p r^p (1-a) - \log \left(a+e^{\pm t r^p}(1-a)\right).
\]
{\color{black} A straightforward calculation shows that this function is maximized at}
\[
t^{\pm}(p,\delta,\nu_0) := \frac{1}{r^p} \log\left(\left(\frac{a(1\pm\delta)^p}{1-(1-a)(1\pm\delta)^p}\right)^{\pm 1}\right).
\]
It is clear that $t^{\pm}(p,\delta,\nu_0)\rightarrow 0$ 
{\color{black}as $p\rightarrow 0$; thus, both}
$\lim\limits_{p \to 0+}\Lambda^+(p,\delta,\nu_0)=0$ and $\lim\limits_{p \to 0+}\Lambda^-(p,\delta,\nu_0)=0$. This suggests that distributions (\ref{eq:example_concentration_breaks}) may contain examples of distributions for which the concentration of $\ell^p$ quasi-norms could be efficiently controlled by choosing $p$.   

A more thorough look into this particular example allows us to rewrite the relevant probabilities as follows:
\[
{\mathbb P}\left(\frac{||{\bf x}||_p}{(n\mu_p)^{1/p}} \geq 1+\delta \right) = {\mathbb P}\left(\frac{1}{n}\sum_{i=1}^n \frac{|x_i|^p}{r^p} - (1-a) \geq ((1+\delta)^p-1) (1-a) \right).
\]
\[
{\mathbb P}\left(\frac{||{\bf x}||_p}{(n\mu_p)^{1/p}} \leq 1-\delta \right) = {\mathbb P}\left(\frac{1}{n}\sum_{i=1}^n \frac{|x_i|^p}{r^p} - (1-a) \leq ((1-\delta)^p-1) (1-a) \right).
\]
{\color{black} Observe that $|x_i|^p/r^{p}$ is a Bernoulli random variable with expectation equal to $1-a$, and let 
\begin{equation}\label{eq:sigma_def}
\sigma^2 := E[||x_i|^p/r^{p}-(1-a)|^2]=a(1-a)
\end{equation}
be its variance, and
\begin{equation}\label{eq:rho_def}
\rho := E[||x_i|^p/r^{p}-(1-a)|^3]=a(1-a)(1-2(1-a)+2(1-a)^2)
\end{equation}
be its centred third moment.
Applying the Berry--Esseen inequality \cite{esseen1956moment} to the normalized sum
\[
S_n:=\frac{1}{\sigma\sqrt n}\sum_{i=1}^n Z_i, \quad \text{where} \quad Z_i := ||x_i|^p/r^{p}-(1-a),
\]
we obtain
\[
\sup_{u\in\mathbb R}\left|\mathbb P(S_n\le u)-\Phi(u)\right|
\le C\,\frac{\rho}{\sigma^3\sqrt n},
\]
where $\Phi$ is the cumulative distribution function for the standard normal distribution, and $C$ is a constant bounded from above and below \cite{esseen1956moment}, \cite{shevtsova2010improvement}:
\[
0.4097\simeq \frac{\sqrt{10}+3}{6\sqrt{2\pi}}\leq C\leq 0.56.
\]
Since
\[
\mathbb P\!\left(\frac{\|x\|_p}{(n\mu_p)^{1/p}}\ge 1+\delta\right)
=
\mathbb P\!\left(
S_n\ge
\frac{\sqrt n}{\sigma}\big((1+\delta)^p-1\big)(1-a)
\right),
\]
and
\[
\mathbb P\!\left(\frac{\|x\|_p}{(n\mu_p)^{1/p}}\le 1-\delta\right)
=
\mathbb P\!\left(
S_n\le
\frac{\sqrt n}{\sigma}\big((1-\delta)^p-1\big)(1-a)
\right),
\]
the Berry--Esseen bound yields
}
\[
{\mathbb P}\left(\frac{||{\bf x}||_p}{(n\mu_p)^{1/p}}\geq 1+\delta \right)\geq \left(1-\Phi\left(\frac{\sqrt{n}}{\sigma}((1+\delta)^p-1)(1-a)\right)\right) - C \frac{\rho}{\sigma^3 \sqrt{n}},
\]
\[
{\mathbb P}\left(\frac{||{\bf x}||_p}{(n\mu_p)^{1/p}}\leq 1-\delta \right)\geq \Phi\left(\frac{\sqrt{n}}{\sigma}((1-\delta)^p-1)(1-a)\right) - C \frac{\rho}{\sigma^3 \sqrt{n}}.
\]
Given that both $(1+\delta)^p-1$ and $(1-\delta)^p-1$, $\delta\in(0,1)$ are strictly monotone with respect to $p$ in $(0,\infty)$, then for any arbitrarily small $\epsilon>0$, any $\delta\in(0,1)$, $\sigma>0$ and any $n\in{\mathbb{N}}$, there exists a $p^\ast(\delta,\epsilon,a,n)>0$: 
\begin{equation}\label{eq:non_concentration_p_choice}
\frac{\sqrt{n}}{\sigma}((1+\delta)^p-1)(1-a)\leq \epsilon,  \ \frac{\sqrt{n}}{\sigma}((1-\delta)^p-1)(1-a)\geq -\epsilon, \ \forall \ p\in(0,p^\ast(\delta,\epsilon,a,n)].
\end{equation}
Therefore, 
\begin{eqnarray}\label{eq:non_concenitration_probability}
& & {\mathbb P}\left(\frac{||{\bf x}||_p}{(n\mu_p)^{1/p}}\geq 1+\delta \right)\geq \left(1-\Phi\left(\epsilon\right)\right) - C \frac{\rho}{\sigma^3 \sqrt{n}}\\
&&{\mathbb P}\left(\frac{||{\bf x}||_p}{(n\mu_p)^{1/p}}\leq 1-\delta \right)\geq \Phi\left(-\epsilon\right) - C \frac{\rho}{\sigma^3 \sqrt{n}}. \nonumber
\end{eqnarray}
for all $p\in(0,p^\ast(\delta,\epsilon,a,n)]$. More formally, the following statement holds:

\begin{proposition}\label{thm:concentration_breaks} There exists an (uncountably large) class ${\cal V}$ of one-dimensional distributions $\nu$ such that the product distribution $\nu^n$ satisfies Assumptions \ref{assume:a} and \ref{assume:b} and the following holds:

\par\smallskip\noindent
\(\forall \nu \in {\cal V}, \,\, \forall \Delta\in(0,1)\,\, \exists N>0: \,\, \forall n>N, \,\, \forall \delta\in(0,1)\,\, \exists q^\ast(\delta,\Delta,\nu,n)>0:\)
\begin{equation}\label{eq:non_concentration}
\mathbb{P}\left(1-\delta \leq \frac{||{\bf x}||_p}{(n\mu_p)^{1/p}}\leq 1+\delta \right)\leq  \Delta \ \mbox{for \ all} \ p\in(0,q^\ast(\delta,\Delta,\nu,n)].
\end{equation}   
\end{proposition}
\begin{proof} Consider distributions $\nu=\nu(a,r)$ defined by (\ref{eq:example_concentration_breaks}). Given that the values  $r>0$ and $a\in(0,1)$ can be chosen arbitrary,  uncountably many such distributions exist. For every such distribution, and fixed $\delta\in(0,1)$, $\epsilon>0$, $a$, $n$, bounds (\ref{eq:non_concenitration_probability}) hold, for an appropriately chosen values $p^{\ast}(\delta,\epsilon,a,n)$. {\color{black} Let $\sigma$ and $\rho$ be as defined in \eqref{eq:sigma_def} and \eqref{eq:rho_def}, respectively. Define
\[
N := 16 C^2 \frac{\rho^2}{\sigma^6 \Delta^2}.
\]
Next choose $\varepsilon>0$ so that
$
\Phi(\varepsilon)=\frac12+\frac{\Delta}{4}.
$
By symmetry of the standard normal distribution,
$
1-\Phi(\varepsilon)=\Phi(-\varepsilon)=\frac12-\frac{\Delta}{4}.
$
Now fix $n>N$ and $\delta\in(0,1)$. By \eqref{eq:non_concentration_p_choice}, there exists
$p^*(\delta,\varepsilon,a,n)>0$ such that for all
$p\in(0,p^*(\delta,\varepsilon,a,n)]$,
\[
\frac{\sqrt n}{\sigma}\big((1+\delta)^p-1\big)(1-a)\le \varepsilon,
\qquad
\frac{\sqrt n}{\sigma}\big((1-\delta)^p-1\big)(1-a)\ge -\varepsilon.
\]
Since $\Phi$ is increasing, \eqref{eq:non_concenitration_probability}  implies
\[
P\!\left(\frac{\|x\|_p}{(n\mu_p)^{1/p}}\ge 1+\delta\right)
\ge
1-\Phi(\varepsilon)-C\frac{\rho}{\sigma^3\sqrt n}
\ge \frac12-\frac{\Delta}{2},
\]
and
\[
P\!\left(\frac{\|x\|_p}{(n\mu_p)^{1/p}}\le 1-\delta\right)
\ge
\Phi(-\varepsilon)-C\frac{\rho}{\sigma^3\sqrt n}
\ge \frac12-\frac{\Delta}{2}.
\]
Therefore,
\[
P\!\left(1-\delta<\frac{\|x\|_p}{(n\mu_p)^{1/p}}<1+\delta\right)\le \Delta.
\]
For fixed $n$, the random variable $\|x\|_p/(n\mu_p)^{1/p}$ takes only finitely many values, so the equalities
\[
\frac{\|x\|_p}{(n\mu_p)^{1/p}}=1\pm\delta
\]
can occur only for finitely many values of $p$. By decreasing $p^*(\delta,\varepsilon,a,n)$ if necessary, we exclude these values and obtain \eqref{eq:non_concentration} with
\[
q^*(\delta,\Delta,\nu,n):=
p^*\!\left(\delta,\Phi^{-1}\!\left(\frac12+\frac{\Delta}{4}\right),a,n\right).
\]}
\end{proof}

Note that the argument used in the proof of Proposition \ref{thm:concentration_breaks} can be trivially extended to cover symmetric product distributions in which the random variables $x_i$ are supported on $\{-r,0,r\}$, $r>0$, and which satisfy:
\[
\mathbb{P}(x_i=0)=a, \ \mathbb{P}(x_i=r)=(1-a)/2, \  \mathbb{P}(x_i=-r)=(1-a)/2, \ a\in(0,1). 
\]

Numerical experiments summarized and presented in Fig. \ref{fig:non_concentration_example} illustrate our theoretical results. 
\begin{figure}
\centering
\includegraphics[width=\textwidth]{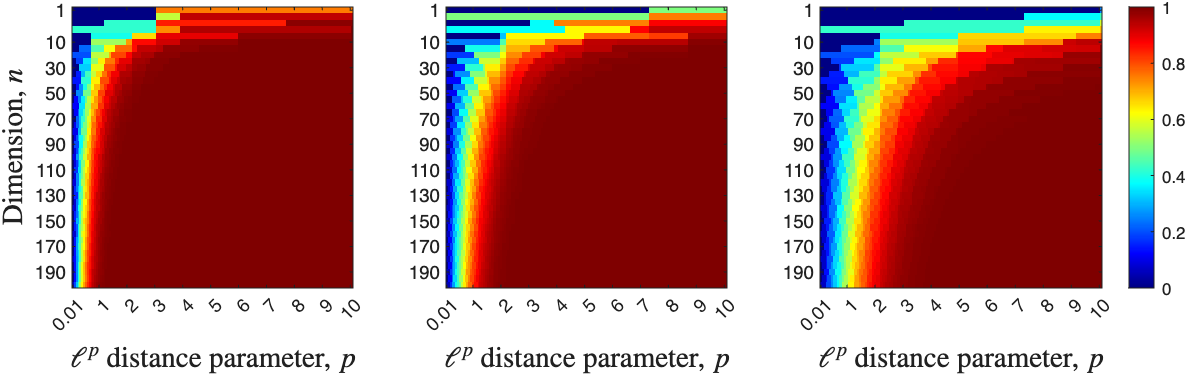}
\caption{Empirical frequencies of the events $1-\delta\leq\frac{||{\bf x}||_p}{(n\mu_p)^{1/p}} \leq 1+\delta$ as functions of $p$ and $n$ for $\delta=0.1$, computed for vectors sampled from the product distribution satisfying (\ref{eq:example_concentration_breaks}) with $a=1/4$ (left panel), $a=1/2$ (centre panel), and $a=3/4$ (right panel),  $r=1$. The size of the data sample was  $5\times 10^5$.}\label{fig:non_concentration_example}
\end{figure}
In Fig. \ref{fig:non_concentration_example} we show empirical evaluations of ${\mathbb P}\left(1-\delta\leq\frac{||{\bf x}||_p}{(n\mu_p)^{1/p}} \leq 1+\delta \right)$ for vectors $\bf x$ sampled independently from the product measure distribution (\ref{eq:example_concentration_breaks}) with $a=1/4$, $a=1/2$, $a=3/4$, $r=1$ for various values of $p$ and $n$, and for $\delta=0.1$. According to this figure, empirically observed frequency of the events $1-\delta\leq {||{\bf x}||_p}/{(n\mu_p)^{1/p}}\leq 1+\delta$ reduces to a small vicinity of $0$ for $p$ small. {\color{black} Experiments for different values of  $a\in(0,1)$ and $r>0$ follow the same pattern, which is in full agreement with what Proposition \ref{thm:concentration_breaks} prescribes}. 

For a slightly modified data-generation distribution ({\ref{eq:example_concentration_breaks}})
\begin{equation}\label{eq:example_concentration}
\mathbb{P}(x_i = \xi) = a, \ \mathbb{P}(x_i = r) = 1 - a, \ a \in (0, 1), \ \xi\in(0,r),
\end{equation}
numerical experiments (see Fig. \ref{fig:concentration_example}) suggest that $\ell^p$ quasi-norm concentration restores even when $\xi$ is close but not exactly equal to $0$.

\begin{figure}
\centering
\includegraphics[width=\textwidth]{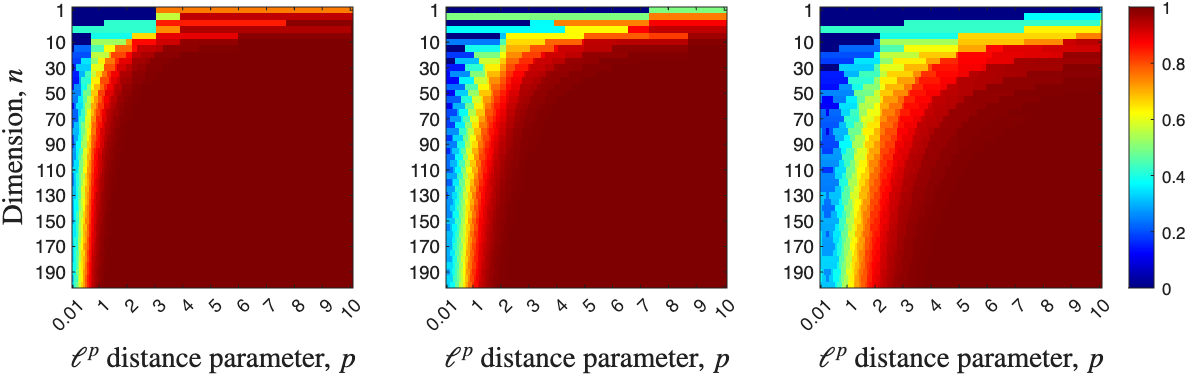} 
\caption{Empirical frequencies of the events $1-\delta\leq\frac{||{\bf x}||_p}{(n\mu_p)^{1/p}} \leq 1+\delta$ as functions of $p$ and $n$ for $\delta=0.1$, computed for vectors sampled from the product distribution satisfying (\ref{eq:example_concentration}) with $a=1/4$ (left panel), $a=1/2$ (centre panel), and $a=3/4$ (right panel),  $r=1$, and $\xi=0.001$. Colour corresponds to the values of empirical frequencies. The size of the data sample  was equal  $5\times 10^5$.}\label{fig:concentration_example}
\end{figure}

A characteristic property of product distributions satisfying (\ref{eq:example_concentration_breaks}) is that they feature {\color{black} an atom} at $0$. Indeed, the probability of drawing a vector whose size is $0$ is $a^n$. This probability is exponentially small (in dimension $n$) but is nevertheless not negligible in terms of its impact on $\ell^p$ quasi-norm concentration {\color{black} for small  $p$, as stated in Proposition \ref{thm:concentration_breaks} and  (\ref{eq:non_concentration})}. 

The class of distributions for which anti-concentration properties (\ref{eq:non_concentration}) hold is not limited to the 
{\color{black} family defined by (\ref{eq:example_concentration_breaks}).} 
In fact, it  includes all distributions that satisfy Assumption \ref{assume:d} below, which is remarkably broad:

\begin{assume}\label{assume:d} There exist positive constants $p_1>0$, $a\in(0,1)$ such that
\begin{equation}\label{eq:gen_breaks}
	\mathbb{P}(x_i=0)=a, \  E[|x_i|^p]<\infty \ \mbox{for all} \ p\in(0,p_1].
\end{equation}
\end{assume}

{\color{black} The second part of Assumption \ref{assume:d} postulates the finiteness of moments for small $p$, and is a significant relaxation of Assumption \ref{assume:b}, which implies that finiteness of moments for all $p$. The main part of Assumption \ref{assume:d} is the condition $\mathbb{P}(x_i=0)=a$ that guarantees a positive probability mass concentrated at $0$.}

In the next theorem we present anti-concentration properties of distributions satisfying Assumption \ref{assume:d}.

\begin{theorem}\label{thm:concentration_breaks:general} 
Suppose that ${\bf x}$ follows the product distribution $\nu^n$ satisfying Assumptions \ref{assume:a} and \ref{assume:d}. Then for any $\Delta \in (0,1)$ and any $\delta\in(0,1)$ there is an $N=N(\Delta, \delta, \nu)\geq 0$ such that for any $n>N$, $n\in\mathbb{N}$ there exists a $p^\ast(n,\Delta,\delta,\nu)$ {\color{black}such that}
\begin{equation}\label{eq:non_concentration:gen:2}
\mathbb{P}\left(1-\delta \leq \frac{||{\bf x}||_p}{(n\mu_p)^{1/p}}\leq 1+\delta \right)\leq  \Delta \quad \mbox{for all} \ p\in(0,p^\ast(n,\Delta,\delta,\nu)].
\end{equation}
Moreover, if parameter $a$ in Assumption \ref{assume:d} is irrational then we can choose $N=0$.
\end{theorem}
\begin{proof}
    Let ${\bf x}=(x_1, \dots, x_n)$ be a random vector satisfying the assumptions, with $a = \mathbb{P}(x_i=0)$. By Assumption 2.6, $a\in(0,1)$. First, we analyse the moments $\mu_p = E[|x_i|^p]$. As $p\to 0+$, $|x_i|^p \to \mathbb{I}(x_i\ne 0)$ almost surely. By Assumption 2.6, $E[|x_i|^{p_1}]<\infty$ for some $p_1>0$. For $p\in(0, p_1]$, the random variables $|x_i|^p$ are dominated by $Y = \max(1, |x_i|^{p_1})$, which is integrable. By the dominated convergence theorem:
$$ 
    \lim_{p\to 0+} \mu_p = E[\lim_{p\to 0+} |x_i|^p] = E[\mathbb{I}(x_i\ne 0)] = 1-a. 
$$

Let $S_{n,p} = ||{\bf x}||_p^p = \sum_{i=1}^n |x_i|^p$. Let $k({\bf x}) = \sum_{i=1}^n \mathbb{I}(x_i\ne 0)$ be the number of non-zero components of ${\bf x}$. Then $S_{n,p} \to k({\bf x})$ almost surely, and thus in probability, as $p\to 0+$ for any fixed $n$. 

 Let $I_{p,n} = [L_{p,n}, R_{p,n}]$, with 
 \[
 L_{p,n} := (1-\delta)^p n\mu_p, \  R_{p,n} := (1+\delta)^p n\mu_p
 \]
 be a concentration interval. We are interested in the upper bound for the concentration probability $P(p,n) = \mathbb{P}(S_{n,p} \in I_{p,n})$.
Let $C_n = n(1-a)$. As $p\to 0+$, both $L_{p,n}$ and $R_{p,n}$ converge to $C_n$. Thus, the interval $I_{p,n}$ shrinks towards the point $\{C_n\}$.


Let $\epsilon>0$. Then, since for any fixed $n$ and the random variable $S_{n,p} - k({\bf x}) \to 0$ in probability as $p\to 0+$,
\begin{itemize}
    \item there exists $p_2(\epsilon, n)>0$, $p_2(\epsilon, n)\leq p_1$ such that for $p\in(0, p_2]$, $\mathbb{P}(|S_{n,p}-k({\bf x})|>\epsilon) < \epsilon$.
    \item there exists $p_3(\epsilon, n)>0$, $p_3(\epsilon, n)\leq p_1$ such that for $p\in(0, p_3]$, $I_{p,n} \subset (C_n-\epsilon, C_n+\epsilon)$.
\end{itemize}

Let $p < \min(p_2, p_3)$. Then
$$
P(p,n) \le \mathbb{P}(S_{n,p} \in I_{p,n} \text{ and } |S_{n,p}-k({\bf x})|\le \epsilon) + \mathbb{P}(|S_{n,p}-k({\bf x})|>\epsilon) 
$$
$$
< \mathbb{P}(|S_{n,p}-C_n| < \epsilon \text{ and } |S_{n,p}-k({\bf x})|\le \epsilon) + \epsilon
$$
$$
\leq \mathbb{P}(|k({\bf x})-C_n| < 2\epsilon) + \epsilon.
$$

Taking the limit superior as $p\to 0+$:
$$ 
\limsup_{p\to 0+} P(p,n) \le \mathbb{P}(|k({\bf x})-C_n| < 2\epsilon) + \epsilon. 
$$
This holds for arbitrary $\epsilon>0$. Taking the limit $\epsilon\to 0+$, we obtain
$$
\limsup_{p\to 0+} P(p,n) \leq \mathbb{P}(k({\bf x})=C_n),
$$
where we used that $k({\bf x})$ is a discrete random variable. Because $k({\bf x})$ has Binomial distribution, 
$$
    \lim_{n\to\infty} \mathbb{P}(k({\bf x})=C_n) = 0.
$$
Hence, for any $\Delta\in(0,1)$ there exists $N\geq 0$ such that for all $n>N$, $\mathbb{P}(k({\bf x})=C_n) < \Delta/2$. Moreover, if $a$ is irrational, then $C_n$ is also irrational, while $k({\bf x})$ is integer, hence $\mathbb{P}(k({\bf x})=C_n) = 0 < \Delta/2$ for all $n$, thus we can choose $N=0$.

Fix $n>N$. Then there exists $p^\ast=p^\ast(n,\Delta,\delta,\nu)>0$ such that for all $p\in(0, p^\ast]$,
$$ 
P(p,n) < \limsup_{p\to 0+} P(p,n) + \Delta/2 \le \mathbb{P}(k({\bf x})=C_n) + \Delta/2 < \Delta/2 + \Delta/2 = \Delta. 
$$
\end{proof}

\begin{remark} \normalfont
Observe that there are uncountably many product distributions satisfying assumptions of Theorem \ref{thm:concentration_breaks:general}. 
\end{remark}

The intuition behind Theorem \ref{thm:concentration_breaks:general} is that, for small $p$, the value of $\ell^p$ quasi-norm is guided by the number of non-zero coordinates of a random vector ${\bf x}$. For vectors with $n$ i.i.d. components satisfying  \eqref{eq:gen_breaks}, the expected number of such non-zero components is $(1-a)n$, but the probability that the number of non-zero components equals to exactly this quantity is  small for $n$ sufficiently large.

{\color{black} Theorem \ref{thm:concentration_breaks:general}, however, should be used with care. If the number $M$ of data points is much smaller than dimension $n$, then each of $M$ random vectors is likely to have a different number of non-zero components, hence $\ell^p$ norms of these points are indeed not concentrated for small enough $p$.  If, however, $M>n$, then, by pigeonhole principle, some vectors must have exactly the same number of non-zero components. Hence, the $\ell^p$ norms of these vectors for small $p$, while not concentrated around the mean value of this norm, will split into at most $n$ clusters. For $M$ much larger than $n$, the number of vectors within these clusters will be large. This grouping of data can, in turn, be viewed as a form of concentration, albeit in a different set.} 

If there is any real number $r$ such that $\mathbb{P}(x_i=r) \in (0,1)$, condition \eqref{eq:gen_breaks} can be ensured by a coordinate shift. 
{\color{black} Since there is a point with positive probability, distributions satisfying \eqref{eq:gen_breaks} are not continuous.
If a continuous or discrete probability distribution satisfies condition \eqref{eq:gen_breaks} approximately, that is, $|x_i|$ is very small with high probability, then the described effect could still be observed numerically for a certain range of parameters (see e.g. Fig. \ref{fig:concentration_example} in which observed empirical frequencies of the events $1-\delta\leq\frac{||{\bf x}||_p}{(n\mu_p)^{1/p}} \leq 1+\delta$ are relatively small for $n\leq 100$ and $p$ around $0.01$).} In the next sections we show that for distributions that have no concentrations near $0$, property \eqref{eq:concentration_necessary} fails, and Theorem \ref{theorem:p-norm-concentration} guarantees uniform concentration for all values of $p$.


As we show below (Section \ref{sec:special_uniform}, Theorem \ref{prop:unip}), behavior of random variables $x_i$ around $0$ appears to be crucial: in absence of concentrations of $x_i$ at $0$ there exist exponential concentration bounds for $\ell^p$ quasi-norms that are independent of $p$. In particular, Theorem \ref{prop:unip} implies that $\ell^p$ quasi-norms concentrate, exponentially in $n$ and uniformly in $p$, for every product distribution (\ref{eq:example_concentration}), albeit with rates that are dependent on the distribution.



\section{Uniform exponential concentration}\label{sec:uniform}

In the previous section, we presented a general concentration estimate for $\ell_p$ quasi-norms {\color{black}(Theorem \ref{theorem:p-norm-concentration})}. We have also provided an example of distributions for which the curse of dimensionality in the context of $\ell^p$ quasi-norm concentration can be avoided or drastically reduced by an appropriate choice of $p$ {\color{black}(Theorem \ref{thm:concentration_breaks:general})}. Here we will 
present conditions when fractional and quasi $p$-norms concentrate exponentially and uniformly in $p$. 

Motivated by the example of non-concentration distributions (\ref{eq:example_concentration_breaks}) from the previous section, we would like to focus on distributions without {\color{black} an atom} at $0$. In particular, we will consider product measure distributions  for which the distribution of $x_i$ satisfies the following additional assumption:
\begin{assume}\label{assume:c} There exists a $y_0>0$ such that
\[
E[|x_i|^{-y_0}]<\infty.
\]
\end{assume}
Assumption \ref{assume:c} guarantees that ${\mathbb P}(x_i=0)=0$, and, more generally, states that $x_i$ is ``not very close to $0$ with high probability''. Assumption \ref{assume:c} implies that 
\[
E[|x_i|^{-y}]<\infty \  \ \ \forall  \ y\in[0,y_0],
\]
because $|x_i|^{-y}\leq\max\{|x_i|^{-y_0},1\}$. 

Assumption \ref{assume:c} also implies that (Jensen's inequality):
\[
E[\log |x_i|]>-\infty.
\]
Similarly, Assumption \ref{assume:b}  implies that
\[
E[\log |x_i|]<+\infty.
\]
Therefore,  Assumption \ref{assume:c} and Assumption \ref{assume:b}, ensure that $E[|\log|x_i||]<+\infty$. Moreover, one can show that $E[|\log|x_i||^q]<+\infty$ for any $q>0$\footnote{Indeed, observe that $|\log(|x_i|^z)|<|x_i|^{-z}$ for any $z>0$ and $|x_i|\in (0,1]$. Therefore, letting $z=y_0/q$, $q>$ we arrive at $(y_0/q)^q \ |\log(|x_i|)|^q<|x_i|^{-y_0}$ for $|x_i|\in(0,1]$. At the same time $|\log(|x_i|)|^q<|x_i|^q$ for $|x_i|>1$. Thus the fact that Assumption \ref{assume:b} and Assumption \ref{assume:c} both hold true implies that $E(|\log(|x_i|)|^q)<\infty$.}.

\subsection{Concentration of {\color{black} quasi} p-norm for small p}\label{sec:small_p}

\subsubsection{General case}

In this section we investigate the behaviour of $\Lambda^\pm(p,\delta,\nu)$ at the limit as $p\to 0+$.

\begin{proposition}\label{prop:pto0} Suppose that ${\bf x}$ follows the product distribution $\nu^n$ satisfying Assumptions \ref{assume:a}, \ref{assume:b}, and
\ref{assume:c} hold. Then, for any $\delta>0$, we have
\begin{equation}\label{eq:liminf0}
\liminf\limits_{p \to 0+}\Lambda^\pm(p,\delta,\nu) \geq f^\pm(\delta,\nu) > 0,
\end{equation}
where $f^-(\delta,\nu):=+\infty$ for $\delta \geq 1$, and otherwise
\begin{equation}\label{eq:fdeltadef}
f^\pm(\delta,\nu) := \sup\limits_{y \geq 0} (\pm y(\log(1\pm\delta)+E[\log |x_i|])-\log E[|x_i|^{\pm y}]).
\end{equation}
\end{proposition}
\begin{proof}
Fix some $\delta>0$, such that $f^\pm(\delta,\nu)$ is given by \eqref{eq:fdeltadef}. For any $y\geq 0$, we have 
\begin{equation}\label{eq:firstineq}
\Lambda^\pm(p,\delta,\nu) = \sup\limits_{t \geq 0} \lambda^\pm(t,p,\delta,\nu) \geq \lambda^\pm(y/p,p,\delta,\nu) = \pm (y/p)(1 \pm \delta)^p \mu_p - \log(E[e^{\pm (y/p)|x_i|^p}]).
\end{equation}
We next calculate the limit of the last expression as $p\to 0+$. 
Assumptions \ref{assume:b} and \ref{assume:c} guarantee that conditions of Proposition \ref{prop:convergence} hold true for $Z=\log((1\pm\delta)|x_i|)$, and \eqref{eq:convergence} implies that
$$
\lim\limits_{p\to 0+}E\left[\frac{(1+\delta)^p |x_i|^p-1}{p}\right]=E[\log|x_i|]+\log(1\pm\delta).
$$
Hence,
\begin{equation}\label{eq:firstlim}
\lim\limits_{p\to 0+}\left(\pm (y/p)(1 \pm \delta)^p \mu_p -(\pm y/p)\right)=\pm y(E[\log|x_i|]+\log(1\pm\delta)).
\end{equation}

Next, we will prove that
\begin{equation}\label{eq:auxlim}
\lim\limits_{p\to 0+}E\left[\exp\left(\pm y \frac{|x_i|^p-1}{p}\right)\right]=E[|x_i|^{\pm y}].
\end{equation}
Indeed, for every $x_i$, expression $\frac{|x_i|^p-1}{p}$ approaches $\log(|x_i|)$ from above as $p\to 0+$. Hence, $\exp\left(y \frac{|x_i|^p-1}{p}\right)$ approaches $\exp(y \log|x_i|)=|x_i|^{y}$ monotonically from above. Then
\[
\exp\left(y \frac{|x_i|^p-1}{p}\right) \leq \exp\left(y \frac{|x_i|^{p_0/2}-1}{p_0/2}\right)
\]
for all $p\in(0,p_0/2]$, but $E\left[\exp\left(y \frac{|x_i|^{p_0/2}-1}{p_0/2}\right)\right]<\infty$ by \eqref{eq:bimplies}, and the ``+'' case of \eqref{eq:auxlim} follows from the dominated convergence theorem. Next, $\exp\left(-y \frac{|x_i|^p-1}{p}\right)$ approaches $|x_i|^{-y}$ monotonically from below as $p\to 0+$, and the ``--'' case of \eqref{eq:auxlim} follows from the monotone convergence theorem.

Taking logarithms of both parts of \eqref{eq:auxlim}, which we can do because $\log$ is a continuous function, we obtain
\begin{equation}\label{eq:secondlim}
\lim\limits_{p\to 0+}\left(\log(E[e^{\pm (y/p)|x_i|^p}])-(\pm y/p)\right)=\log E[|x_i|^{\pm y}].
\end{equation}
Subtracting this from \eqref{eq:firstlim}, we get
$$
\lim\limits_{p\to 0+} \lambda^\pm(y/p,p,\delta,\nu) = \pm y(\log(1\pm\delta)+E[\log |x_i|])-\log E[|x_i|^{\pm y}].
$$
Let $g^\pm(y,\delta,\nu)$ 
{\color{black} denote the right-hand side of the previous equation.}
From \eqref{eq:firstineq},
$$
\liminf\limits_{p \to 0+}\Lambda^\pm(p,\delta,\nu) \geq \lim\limits_{p\to 0+} \lambda^\pm(y/p,p,\delta,\nu) = g^\pm(y,\delta,\nu).
$$
Because $y\geq 0$ was arbitrary, this implies that
$$
\liminf\limits_{p \to 0+}\Lambda^\pm(p,\delta,\nu) \geq \sup\limits_{y \geq 0} g^\pm(y,\delta,\nu) = f^\pm(\delta,\nu).
$$
To finish the proof of the first inequality in \eqref{eq:liminf0}, it remains to prove $\liminf\limits_{p \to 0+}\Lambda^-(p,\delta,\nu) \geq f^-(\delta,\nu)$ for $\delta\geq 1$. For $\delta>1$, this is trivial because $\Lambda^-(p,\delta,\nu)=+\infty$ by definition. Finally, for $\delta=1$, and any fixed $y>0$,
$$
\liminf\limits_{p \to 0+}\Lambda^-(p,1,\nu) \geq \lim\limits_{p\to 0+} \lambda^-(y/p,p,1,\nu)= \lim\limits_{p\to 0+}\left(-\log(E[e^{- (y/p)|x_i|^p}])\right) = +\infty, 
$$
where the last equality follows from \eqref{eq:secondlim}.

We now prove the second inequality \eqref{eq:liminf0}, that is, that $f^\pm(\delta,\nu)>0$ for any $\delta>0$. Let 
$$
g^{\pm}(y,\delta,\nu) := \pm y(\log(1\pm\delta)+E[\log |x_i|])-\log E[|x_i|^{\pm y}]
$$ 
be the expression under the supremum in \eqref{eq:fdeltadef}.

Assumptions \ref{assume:b} and \ref{assume:c} guarantee that conditions of Proposition \ref{prop:convergence} hold true for $Z=\pm\log|x_i|$, and \eqref{eq:convergence} implies that
$$
\lim\limits_{y\to 0+}E\left[\frac{|x_i|^{\pm y}-1}{y}\right]=\pm E[\log|x_i|].
$$
Hence
$$
\lim\limits_{y\to 0+}\frac{\log E[|x_i|^{\pm y}]}{\pm y}=\lim\limits_{y\to 0+}\frac{\log E[|x_i|^{\pm y}]}{E[|x_i|^{\pm y}]-1}\lim\limits_{y\to 0+}\frac{E[|x_i|^{\pm y}]-1}{\pm y} = 1 \cdot E[\log|x_i|] = E[\log|x_i|],
$$
and
\begin{equation}\label{eq:gto0}
\lim\limits_{y\to 0+}\frac{g^\pm(y,\delta,\nu)}{y}=\pm \log(1 \pm \delta) \pm (E[\log|x_i|] - E[\log|x_i|]) = \pm \log(1 \pm \delta) > 0.
\end{equation}
Because $g^\pm(0,\delta,\nu)=0$, \eqref{eq:gto0} implies that 
$f^\pm(\delta,\nu) = \sup\limits_{y \geq 0} g^\pm(y,\delta,\nu)>0$.
\end{proof}

\subsubsection{Concentration of quasi p-norms for some common distributions of $x_i$ (small p)}

Our next results present the values of $f^\pm(\delta,\nu)$ for two common and practically relevant distributions. We begin with establishing a bound on the exponential concentration rate for data sampled from uniform distributions in the cubes $[-1,1]^n$, $n\in\mathbb{N}$.

\begin{proposition}\label{prop:cube}
Let $\nu$ be the uniform distribution on $[-1,1]$, so that ${\bf x}$ is uniformly distributed in cube $[-1,1]^n$, {\color{black} and let $\delta\in(0,1)$.} Then $\mu_p=\frac{1}{1+p}$, and \eqref{eq:fdeltadef} reduces to
$$
f^\pm(\delta,\nu) = -\log[(1\pm\delta)(1-\log(1\pm\delta))].
$$
\end{proposition}
\begin{proof}
In this case, $E[\log |x_i|]=-1$ and \eqref{eq:fdeltadef} implies that
$$
f^+(\delta,\nu) = \sup\limits_{y \geq 0} \left(y(\log(1+\delta)-1)-\log \frac{1}{1+y}\right).
$$
Let $g(y)$ be the expression under supremum. Then $g'(y)=\log(1+\delta)-1+\frac{1}{1+y}=0$ when $y=\frac{\log(1+\delta)}{1-\log(1+\delta)}$. It is easy to check that this is the global maximum, and 
$$
f^+(\delta,\nu) = \left(\frac{\log(1+\delta)}{1-\log(1+\delta)}(\log(1+\delta)-1)-\log (1-\log(1+\delta))\right) 
$$
$$
= -\log[(1+\delta)(1-\log(1+\delta))].
$$

For $f^-(\delta,\nu)$, \eqref{eq:fdeltadef} implies that
$$
f^-(\delta,\nu) = \sup\limits_{y \geq 0} (-y(\log(1-\delta)-1-\log E[|x_i|^{-y}]) > 0.
$$
For $y\geq 1$, $E[|x_i|^{-y}]=+\infty$, and the expression under supremum is $-\infty$. Hence, the supremum is over $y\in(0,1)$, in which case $E[|x_i|^{-y}]=\frac{1}{1-y}$. Then 
$$
f^-(\delta,\nu) = \sup\limits_{y\in(0,1)}g(y), \quad g(y)=-y(\log(1-\delta)-1)-\log \frac{1}{1-y}.
$$ 
Then $g'(y)=1-\log(1-\delta)-\frac{1}{1-y}=0$ when $y=\frac{-\log(1-\delta)}{1-\log(1-\delta)}$. It is easy to check that this is the global maximum, and 
$$
f^-(\delta,\nu) = \left(\frac{\log(1-\delta)}{1-\log(1-\delta)}(\log(1-\delta)-1)-\log (1-\log(1-\delta))\right) 
$$
$$
= -\log[(1-\delta)(1-\log(1-\delta))].
$$
\end{proof}

Another case of practical interest is the scenario in which we look at the distances between random pairs of points rather than the lengths of individual data vectors. We show that if both data vectors are independent samples from the uniform distribution in an $n$-cube then the $\ell^p$ lengths of these differences inevitably and exponentially (uniformly in $p$) concentrate with dimension $n$ too.

\begin{proposition}\label{prop:cube_diff}
Let ${\bf x}={\bf y}-{\bf z}$, where ${\bf y}$ and ${\bf z}$ are independent and uniformly distributed in cube $[-1,1]^n$. Let $\nu$ be the distribution of $x_i$. Then the density of $|x_i|=|y_i-z_i|$ is $1-x/2$, $0\leq x \leq 2$, 
\[
\mu_p=\int_0^2 x^p(1-x/2)dx=\frac{2^{1+p}}{2+3p+p^2},
\]
and \eqref{eq:fdeltadef} reduces to
\begin{equation}\label{eq:distfdef}
f^\pm(\delta,\nu) = \frac{5}{4}-\frac{3}{2}\log(1\pm\delta)-\frac{1}{4}\sqrt{h^\pm(\delta)}-\log(-4+\sqrt{h^\pm(\delta)}),
\end{equation}
where $h^\pm(\delta)=25-12\log(1\pm\delta)+4\log^2(1\pm\delta)$.
\end{proposition}
\begin{proof}
In this case, $E[\log |x_i|]=\int_0^2 (1-x/2)\log(x)dx=-\frac{3}{2}+\log 2$ for $p>-1$ and
$$
f^+(\delta,\nu) = \sup\limits_{y \geq 0} g^+(y), \quad f^-(\delta) = \sup\limits_{y\in(0,1)} g^-(y),
$$
where
\[
\begin{array}{ll}
g^\pm(y) =& \pm y(\log(1\pm \delta)-\frac{3}{2}+\log 2)-\log \frac{2^{1\pm y}}{2 \pm 3y+y^2} \\
&= \pm y(-3/2+\log(1\pm\delta))-\log 2 + \log(2\pm 3y+y^2). 
\end{array}
\]
Then the derivative
$$
\frac{d g^\pm(y)}{dy}=\pm(-3/2+\log(1\pm\delta))+\frac{\pm 3+2y}{2\pm 3y+y^2}
$$
has a root in the required range
$$
y_*^\pm := \frac{-5+6\log(1\pm\delta)+\sqrt{h^\pm(\delta)}}{\pm 6 \mp 4\log(1\pm\delta)},
$$
and $f^\pm(\delta,\nu) = g(y_*^\pm)$, which simplifies to \eqref{eq:distfdef}.
\end{proof}

\subsection{Concentration of p-norm for large p for bounded distributions}\label{sec:large_p}

If the component distribution $\nu$ of ${\bf x}$ has bounded support, Theorem \ref{theorem:p-norm-concentration} holds for all $p>0$. In this section, we investigate the limiting behaviour of $\Lambda^\pm(p,\delta,\nu)$ as $p\to\infty$. The need to consider this limiting case is rather nuanced. Indeed, the behavior of p-norms at $p=\infty$ can be computed in a straightforward manner. However, as we later show, considering the limiting behavior of $\Lambda^\pm(p,\delta,\nu)$ as $p\to\infty$ allows us to prove continuity of $\Lambda^\pm(\cdot,\delta,\nu)$ on {\color{black}$(0,\infty]$}. The latter property is useful for establishing uniform concentration bounds: bounds which are independent on $p$.

\begin{proposition}\label{prop:infty}
Suppose that ${\bf x}$ follows the product distribution $\nu^n$ satisfying Assumptions \ref{assume:a}, \ref{assume:b}, and \ref{assume:c}, and let {\color{black} $B:=\mathrm{ess}\sup |x_i|<\infty$.}
Let $F(x)={\mathbb P}(|x_i| \leq x)$ be the cumulative distribution function of $|x_i|$. Then, for any $\delta>0$,
\begin{equation}\label{eq:infty+}
\lim\limits_{p\to\infty}\Lambda^+(p,\delta,\nu) = +\infty,
\end{equation}
and, for any $\delta\in(0,1)$, {\color{black}defining $F((1-\delta)B-) := \lim\limits_{\epsilon\to 0+}F((1-\delta)B-\epsilon)$,} 
\begin{equation}\label{eq:infty-}
-\log(F((1-\delta)B-)) \geq \limsup\limits_{p\to\infty}\Lambda^-(p,\delta,\nu) \geq \liminf\limits_{p\to\infty}\Lambda^-(p,\delta,\nu) \geq -\log(F((1-\delta)B))>0.
\end{equation}
In particular, if ${\mathbb P}(|x_i|=(1-\delta)B)=0$, then
$$
\lim\limits_{p\to\infty}\Lambda^-(p,\delta,\nu) = -\log(F((1-\delta)B)).
$$
\end{proposition} 
\begin{proof}
Let us prove \eqref{eq:infty+}. Choose any constant $C$ such that $\frac{B}{1+\delta}<C<B$. Let $\beta=\beta(C)={\mathbb P}(|x_i|>C)>0$. 
%
Let $y^-_i$ be random variable such that (i) $y^-_i=C$ whenever $|x_i|>C$, and (ii) $y^-_i=0$ otherwise. Then $B \geq |x_i| \geq y^-_i$ almost surely. Hence $\mu_p=E[|x_i|^p]\geq E[(y^-_i)^p]=C^p\beta$, $\log E[e^{t|x_i|^p}]\leq \log e^{tB^p} =tB^p$, and 
$$
\lambda^+(t,p,\delta,\nu)=t(1+\delta)^p \mu_p - \log(E[e^{t|x_i|^p}]) \geq t((1+\delta)^p C^p\beta - B^p). 
$$ 
Because $\frac{B}{(1+\delta)C}<1$, we have $\left(\frac{B}{(1+\delta)C}\right)^p<\beta$ for all sufficiently large $p$. For such $p$, $(1+\delta)^p C^p\beta - B^p>0$, and 
$$
\Lambda^+(p,\delta,\nu) = \sup\limits_{t \geq 0} \lambda^+(t,p,\delta,\nu) \geq \sup\limits_{t \geq 0} t((1+\delta)^p C^p\beta - B^p) = +\infty. 
$$

Next let us prove the lower bound in \eqref{eq:infty-}. This time let $C$  be any constant such that $B(1-\delta)<C<B$, and let $\beta$ and $y^-_i$ be as above. Then
$$
-\log E[e^{-t|x_i|^p}] \geq -\log E[e^{-t|y^-_i|^p}]=-\log(e^{0}(1-\beta)+e^{-t C^p}\beta)
$$
and
$$
\lambda^-(t,p,\delta,\nu)=-t(1-\delta)^p \mu_p - \log(E[e^{-t|x_i|^p}]) 
\geq -t(1-\delta)^p B^p-\log((1-\beta)+e^{-t C^p}\beta). 
$$
Substituting $t=\frac{1}{y^p}$ for any $(1-\delta)B < y < C$, we get
$$
\Lambda^-(p,\delta,\nu) = \sup\limits_{t \geq 0} \lambda^-(t,p,\delta,\nu) \geq  -\left(\frac{(1-\delta)B}{y}\right)^p-\log((1-\beta)+e^{-(C/y)^p}\beta).
$$
Because $\frac{(1-\delta)B}{y}<1$ and $\frac{C}{y}>1$, this implies that
$$
\liminf\limits_{p\to\infty}\Lambda^-(p,\delta,\nu) \geq -\log(1-\beta).
$$
Because $1-\beta={\mathbb P}(|x_i|\leq C)=F(C)$, and $C>B(1-\delta)$ was arbitrary, this implies that
$$
\liminf\limits_{p\to\infty}\Lambda^-(p,\delta,\nu) \geq \lim\limits_{C\to B(1-\delta)+}(-\log F(C)) = -\log F((1-\delta)B),
$$
where the last equality follows from right-continuity of the cdf $F$.

Finally, to prove the upper bound in \eqref{eq:infty-}, choose any $C$ such that $0<C<(1-\delta)B$, set $\beta={\mathbb P}(|x_i|>C)$ as above, and consider random variable $y^+_i$ such that (i) $y^+_i=+\infty$ whenever $|x_i|>C$, and (ii) $y^+_i=C$ otherwise. Then $y^+_i \geq x_i$ almost surely, and 
$$
-\log E[e^{-t|x_i|^p}] \leq -\log E[e^{-t|y^+_i|^p}]=-\log(e^{-t C^p}(1-\beta))=t C^p-\log(1-\beta).
$$
Hence, 
$$
\lambda^-(t,p,\delta,\nu)=-t(1-\delta)^p \mu_p - \log(E[e^{-t|x_i|^p}]) \leq t(-(1-\delta)^p \mu_p + C^p)-\log(1-\beta).
$$ 
Because $\lim\limits_{p\to\infty}(\mu_p)^{1/p}=B$, we have, for all sufficiently large $p$, $(\mu_p)^{1/p}>\frac{C}{1-\delta}$. Then $-(1-\delta)^p \mu_p + C^p < 0$, and 
$$
\Lambda^-(p,\delta,\nu) = \sup\limits_{t\geq 0} \lambda^-(t,p,\delta,\nu) \leq \sup\limits_{t\geq 0}(t(-(1-\delta)^p \mu_p + C^p)-\log(1-\beta))= -\log(1-\beta).
$$
$$
\limsup\limits_{p\to\infty}\Lambda^-(p,\delta,\nu) \leq -\log(1-\beta).
$$
Because $1-\beta={\mathbb P}(|x_i|\leq C)=F(C)$, and $C<B(1-\delta)$ was arbitrary, {\color{black} this implies that
$$
\limsup\limits_{p\to\infty}\Lambda^-(p,\delta,\nu) \leq -\log\left(\lim\limits_{\epsilon\to 0+}F((1-\delta)B-\epsilon)\right) = -\log(F((1-\delta)B-)),
$$
which is the upper bound in \eqref{eq:infty-}.}
\end{proof}

Motivated by Proposition \ref{prop:infty}, we can define, by continuity,
\[
\Lambda^+(+\infty,\delta,\nu):=+\infty, \ \Lambda^-(+\infty,\delta,\nu):=-\log(F((1-\delta)B)),
\]
and note that the inequalities \eqref{eq:genchern} remains correct for $p=\infty$. In this case, $||{\bf x}||_\infty=\max_i |x_i|$, $(\mu_p)^{1/p}=B$, $n^{1/p}=n^0=1$, and the inequalities take the form
\begin{equation}\label{eq:cherninfty}
{\mathbb P}\left(\frac{||{\bf x}||_\infty}{B} \geq 1+\delta \right) = 0, 
\quad 
{\mathbb P}\left(\frac{||{\bf x}||_\infty}{B} \leq 1-\delta \right) \leq F((1-\delta)B)^n. 
\end{equation}

\subsection{The uniform concentration for all p}\label{sec:special_uniform}

Let Assumptions \ref{assume:a}, \ref{assume:b}, and \ref{assume:c} hold. If ${\bf x}$ has unbounded support, let $p_0 \in (0,+\infty)$ be an arbitrary constant satisfying Assumption \ref{assume:b}. If ${\bf x}$ has bounded support, define $p_0=+\infty$.


\begin{theorem}\label{prop:unip}
Suppose that ${\bf x}$ follows the product distribution $\nu^n$ satisfying Assumptions \ref{assume:a}, \ref{assume:b}, 
 and \ref{assume:c}. Then
\begin{equation}\label{eq:infdef}
f_*^\pm(\delta,\nu) := \inf\limits_{p\in(0,p_0]}\Lambda^\pm(p,\delta,\nu)
\end{equation} 
are strictly positive functions for any $\delta\in(0,1)$, and inequalities
\begin{equation}\label{eq:unip}
{\mathbb P}\left(\frac{||{\bf x}||_p}{(n\mu_p)^{1/p}} \geq 1+\delta \right) \leq \exp[-n \cdot f_*^+(\delta,\nu)], 
\quad 
{\mathbb P}\left(\frac{||{\bf x}||_p}{(n\mu_p)^{1/p}} \leq 1-\delta \right) \leq \exp[-n \cdot f_*^-(\delta,\nu)], 
\end{equation}
hold for all $n \geq 1$ and all $p\in(0,p_0]$. In particular,
\begin{equation}\label{eq:unip2}
{\mathbb P}\left(1-\delta < \frac{||{\bf x}||_p}{(n\mu_p)^{1/p}} < 1+\delta \right) \geq 1-2\exp[-n \cdot f^*(\delta,\nu)], 
\quad 
\forall n \geq 1, \, \forall p\in(0,p_0]
\end{equation}
where $f^*(\delta,\nu)=\min\{f_*^+(\delta,\nu),f_*^-(\delta,\nu)\}>0$.
\end{theorem}
\begin{proof}
For every fixed $\nu$ and fixed $\delta\in(0,1)$, consider $\Lambda^\pm(p,\delta,\nu)$ as functions of $p$ on $(0,p_0]$. Because they are defined as supremum of continuous functions, $\Lambda^\pm(p,\delta,\nu)$ are lower semicontinuous. 
 Hence, the infimum in \eqref{eq:infdef} in either attained at the internal point $p^*\in(0,p_0]$, or is approaching when $p\to 0+$ or $p\to p_0-$. If $f_*^\pm(\delta,\nu)=\Lambda^\pm(p^*,\delta,\nu)$, then $f_*^\pm(\delta,\nu)>0$ by Theorem \ref{theorem:p-norm-concentration}. If $f_*^\pm(\delta,\nu)=\liminf\limits_{p \to 0+}\Lambda^\pm(p,\delta,\nu)$, then $f_*^\pm(\delta,\nu)>0$ by Proposition \ref{prop:pto0}. Now assume that $f_*^\pm(\delta,\nu)=\liminf\limits_{p \to p_0-}\Lambda^\pm(p,\delta,\nu)$. If ${\bf x}$ has unbounded support, and $p_0 \in (0,+\infty)$, then $f_*^\pm(\delta,\nu)=\Lambda^\pm(p_0,\delta,\nu)>0$ by Theorem \ref{theorem:p-norm-concentration}. Finally, if ${\bf x}$ has bounded support, and $p_0=\infty$, then $f_*^\pm(\delta,\nu)=\liminf\limits_{p \to \infty}\Lambda^\pm(p,\delta,\nu)>0$ by Proposition \ref{prop:infty}. The bounds \eqref{eq:unip} and \eqref{eq:unip2} follow from \eqref{eq:infdef} and Theorem \ref{theorem:p-norm-concentration}.
\end{proof}

The Proposition below derives an explicit formula for $f_*^+(\delta,\nu)$ for the uniform distribution in the centred cube.


\begin{proposition}\label{prop:cubemon}
Let $\nu$ be the uniform distribution on $[-1,1]$, so that ${\bf x}$ is uniformly distributed in cube $[-1,1]^n$. Then, for any fixed {\color{black} $\delta\in(0,1)$}, $\Lambda^+(p,\delta,\nu)$ is an increasing function of $p$ on $(0,\infty)$. In particular, $f_*^+(\delta,\nu)=f^+(\delta,\nu)$, where $f^+(\delta,\nu)$ {\color{black} and $\mu_p$ are} described in Proposition \ref{prop:cube}.  
\end{proposition}
\begin{proof}
It is clear that for any $C(p,\delta)>0$, we have
$$
\Lambda^+(p,\delta) = \sup\limits_{t \geq 0} \lambda^+(t,p,\delta,\nu) = \sup\limits_{t \geq 0} \lambda^+(t C(p,\delta),p,\delta,\nu).
$$
{\color{black} Put $C(p,\delta):=(\mu_p((1+\delta)^p -1))^{-1}$ and denote $\lambda_1^+(t,p,\delta):=\lambda^+(t C(p,\delta),p,\delta,\nu)$.}  
In this case, \eqref{eq:tto0} implies that $\lim\limits_{t\to 0+}\frac{\lambda_1^+(t,p,\delta)}{t}=1.$ It suffices to prove that $\lambda_1^+(t,p,\delta)$ is non-decreasing in p for fixed $t$ and $\delta$. 
By definition \eqref{eq:lambdadef},
$$
\lambda_1^+(t,p,\delta) = t\frac{(1+\delta)^p}{(1+\delta)^p -1}- \log\int_0^1 \exp\left(t\frac{1+p}{(1+\delta)^p -1}x^p\right)dx,
$$ 
where we have used that $\mu_p=\frac{1}{1+p}$. Denote by $h(p,x):=\frac{1+p}{(1 \pm \delta)^p -1}x^p$ {\color{black} the expression multiplying $t$ in the exponent.
It suffices to prove that the}
derivative 
$$
\frac{\partial \lambda_1^+(t,p,\delta)}{\partial p}=t\frac{-(1+ \delta)^p\log(1+\delta)}{((1+\delta)^p -1)^2}-\int_0^1 e^{ h(p,x)t}\frac{\partial h(p,x)}{\partial p}t dx\left(\int_0^1 e^{h(p,x)t}dx\right)^{-1}
$$
is non-negative, {\color{black} or, equivalently, that}
\begin{equation}\label{eq:cubeaux}
\int_0^1 u(p,x)e^{h(p,x)t}dx \geq 0,
\quad
u(p,x):=\frac{-(1+\delta)^p\log(1+\delta)}{((1+\delta)^p -1)^2}-\frac{\partial h(p,x)}{\partial p}.
\end{equation}
Let $I(p,x):=\int\limits_0^x u(p,y)dy$. We will prove that $I(p,x)\leq 0$ for all $x\in[0,1]$, with $I(p,0)=I(p,1)=0$. Then integration by parts will imply that
$$
\int_0^1 u(p,x)e^{h(p,x)t}dx = -\int_0^1 I(p,x)\frac{\partial h(p,x)}{\partial x}e^{h(p,x)t}dx \geq 0,
$$
and \eqref{eq:cubeaux} follows. Direct calculation shows that the desired inequality $I(p,x)\leq 0$ is equivalent to
$$
(1+\delta)^p(x^p-1)\log(1+\delta)-((1+\delta)^p-1)x^p\log(x) \leq 0,
$$
which in turn is equivalent to $v(1+\delta)\geq v(x)$, where $v(y)=\frac{y^p\log y}{y^p-1}$. Because $1+\delta>x$, it suffices to prove that $v(y)$ is a non-decreasing function. The inequality $v'(y)\geq 0$ is equivalent to $-1+y^p-p\log y \geq 0$. Consider the last expression as a function of $p$: $g(p)=-1+y^p-p\log y$. Then $g(0)=0$, and $g'(p)=y^p\log y-\log y=\log y(y^p-1)\geq 0$. Hence $g(p)$ is non-decreasing and $g(p) \geq g(0)=0$, which finishes the proof.
\end{proof}

Theorem \ref{prop:unip} and Propositions \ref{prop:cubemon}, and \ref{prop:cube} imply that, for ${\bf x}$ being uniformly distributed in cube $[-1,1]^n$, the inequality 
\begin{equation}\label{eq:cubeupper}
{\mathbb P}\left(\frac{||{\bf x}||_p}{\left(\frac{n}{1+p}\right)^{1/p}} \geq 1+\delta \right)\leq [(1+\delta)(1-\log(1+\delta))]^n
\end{equation}
holds for any $n\geq 1$, $p>0$, and {\color{black} $\delta\in(0,1)$}. Note that the right-hand side does not depend on $p$.

{\color{black} For small $n$, the right-hand side of \eqref{eq:cubeupper} is close to $1$, hence inequality \eqref{eq:cubeupper} does not provide a useful information. However, the strength of \eqref{eq:cubeupper} is that its right-hand side decays exponentially fast when the dimension grows. As an illustration, for $\delta=0.2$ the right-hand side of \eqref{eq:cubeupper} reduces to $(0.981...)^n$. This bound is useless for small $n$, but becomes $\approx 0.0225$ for $n=200$ and $\approx 5.8 \cdot 10^{-9}$ for $n=1000$.}

\subsection{Concentration of p-norm for small $\delta$}\label{sec:small_delta}

For the vast majority of distributions $\nu$, functions $\Lambda^\pm(p,\delta,\nu)$ in \eqref{eq:lambdadef} can in general be computed only numerically. However, we are usually interested in concentration inequality \eqref{eq:genchern} 
{\color{black} only for small $\delta$.} 
This section investigates the limiting behaviour of $\Lambda^\pm(p,\delta,\nu)$ as $\delta \to 0+$.  

\subsubsection{Case 1: any $p>0$ and small $\delta$}

\begin{proposition}\label{prop:deltato0}
Suppose that ${\bf x}$ follows the product distribution $\nu^n$ satisfying 
{\color{black} Assumptions \ref{assume:a} and \ref{assume:b}.} 
Then, for any {\color{black} $p\in(0,p_0]$},
\begin{equation}\label{eq:deltato0}
\lim\limits_{\delta \to 0+}\frac{\Lambda^\pm(p,\delta,\nu)}{\delta^2} = \frac{p^2}{2}\frac{E[|x_i|^p]^2}{Var[|x_i|^p]}.
\end{equation}
\end{proposition}
\begin{proof}
Let {\color{black} $p\in(0,p_0]$} be fixed, and let $\lambda^\pm(t,p,\delta,\nu)$ be as defined in  \eqref{eq:lambdadef}. 
Then
$$
\Lambda^\pm(p,\delta,\nu) = \sup\limits_{t \geq 0} \lambda^\pm(t,p,\delta,\nu) = \sup\limits_{t \geq 0}(at - h(t)),
$$
where $a=\pm (1 \pm \delta)^p \mu_p \geq \pm \mu_p$ and $h(t):=\log(E[e^{\pm t|x_i|^p}])$. Function $h(t)$ is the logarithm of the moment generating function of random variable $Z=\pm|x_i|^p$, and is therefore a convex function. It is a twice differentiable with derivative
$$
h'(t) = \frac{E\left[\pm |x_i|^p e^{\pm t|x_i|^p}\right]}{E\left[e^{\pm t|x_i|^p}\right]} 
$$
and second derivative
$$
h''(t) = \frac{1}{E\left[e^{\pm t|x_i|^p}\right]^2}\left(E\left[|x_i|^{2p} e^{\pm t|x_i|^p}\right]E\left[e^{\pm t|x_i|^p}\right] - E\left[\pm |x_i|^p e^{\pm t|x_i|^p}\right]^2\right).
$$
In particular, $h(0)=0$, $h'(0)=\pm \mu_p$, and $h''(0)=E[|x_i|^{2p}]-E[|x_i|^{p}]^2 = Var[|x_i|^p]$.

Because $h(t)$ is convex, $h'(t)$ is non-decreasing, and, for any $a\geq \pm\mu_p$ and $t\leq 0$, 
$$
(at-h(t))'=a-h'(t) \geq a-h'(0) = a-(\pm\mu_p) \geq 0,
$$
hence $at-h(t)$ is non-decreasing on $(-\infty,0]$ and $at-h(t)\leq a\cdot 0 - h(0)=0$. Hence
$$
\Lambda^\pm(p,\delta,\nu) = \sup\limits_{t \geq 0}(at - h(t)) = \sup\limits_{t \in {\mathbb R}}(at - h(t)).
$$

{\color{black} Let $g(a):=\sup_{t \in {\mathbb R}}(at - h(t))$ be the \emph{Legendre transform} of $h$ \cite[Section 12]{Rockafellar1970}.}
For $a=\pm\mu_p=h'(0)$, $(at-h(t))'=0$ at $t=0$, and $g(\pm\mu_p)=\pm\mu_p \cdot 0 - h(0) = 0$.
Further, existence of derivatives of $h$ imply that $g(a)$ is also twice differentiable, and its derivatives satisfy $g'(h'(t))=t$ and $g''(a)=\frac{1}{h''(g'(a))}$. In particular, $g'(\pm \mu_p)=g'(h'(0))=0$, and
$$
g''(\pm \mu_p) = \frac{1}{h''(g'(\pm \mu_p))} = \frac{1}{h''(0)}=\frac{1}{Var[|x_i|^p]}.
$$ 
Hence,
$$
\lim\limits_{a \to \pm \mu_p}\frac{g(a)}{(a-(\pm \mu_p))^2} = \frac{1}{2}g''(\pm \mu_p) = \frac{1}{2 Var[|x_i|^p},
$$ 
and
$$
\lim\limits_{\delta \to 0+}\frac{\Lambda^\pm(p,\delta,\nu)}{\delta^2} = 
\lim\limits_{\delta \to 0+}\frac{g(\pm (1 \pm \delta)^p \mu_p)}{\mu_p^2((1 \pm \delta)^p-1)^2} \cdot 
\lim\limits_{\delta \to 0+}\frac{\mu_p^2((1 \pm \delta)^p-1)^2}{\delta^2} = \frac{p^2}{2}\frac{E[|x_i|^p]^2}{Var[|x_i|^p]}.
$$
\phantom{.}
\end{proof}

Theorem \ref{theorem:p-norm-concentration} and Proposition \ref{prop:deltato0} imply the following p-norm concentration inequality for small $\delta$-s. 

\begin{proposition}\label{prop:cherndelta}
Suppose that ${\bf x}$ follows the product distribution $\nu^n$ satisfying Assumptions \ref{assume:a} and \ref{assume:b}. For any $p\in(0,p_0]$, and any $\epsilon>0$, there is a constant $\delta_0=\delta_0(p,\epsilon,\nu)>0$, such that inequalities 
\begin{equation}\label{eq:cherndelta}
\begin{array}{l}
{\mathbb P}\left(\frac{||{\bf x}||_p}{(n\mu_p)^{1/p}} \geq 1+\delta \right) \leq \exp[-n \cdot (\phi(p,\nu)-\epsilon)\delta^2]\\ 
\\
{\mathbb P}\left(\frac{||{\bf x}||_p}{(n\mu_p)^{1/p}} \leq 1-\delta \right) \leq \exp[-n \cdot (\phi(p,\nu)-\epsilon)\delta^2],
\end{array}
\end{equation}
hold for all $0<\delta \leq \delta_0$ and all $n\geq 1$, where 
\begin{equation}\label{eq:phidef}
\phi(p,\nu) := \frac{p^2}{2}\frac{E[|x_i|^p]^2}{Var[|x_i|^p]}.
\end{equation}
\end{proposition}

Note that the bound in \eqref{eq:cherndelta} is the same for the upper and lower tails. In particular, \eqref{eq:cherndelta} implies the two-sided inequality
\begin{equation}\label{eq:twosided}
{\mathbb P}\left(1-\delta < \frac{||{\bf x}||_p}{(n\mu_p)^{1/p}} < 1+\delta \right) \geq 1-2\exp[-n \cdot (\phi(p,\nu)-\epsilon)\delta^2].
\end{equation}

\begin{remark}\label{rem:special_cases} The function $\phi(p,\nu)$ may be explicitly computed for many distributions $\nu$ of interest. In particular,
\begin{itemize}
\item[(a)] If ${\bf x}$ is uniformly distributed in cube $[-1,1]^n$, then 
$
\phi(p,\nu)=\frac{1}{2}+p
$
and \eqref{eq:twosided} reduces to
$$
{\mathbb P}\left(1-\delta < \frac{||{\bf x}||_p}{(n\mu_p)^{1/p}} < 1+\delta \right) \geq 1-2\exp[-n \cdot (0.5+p-\epsilon)\delta^2].
$$
\item[(b)] If ${\bf x}={\bf y}-{\bf z}$, where ${\bf y}$ and ${\bf z}$ are independent and uniformly distributed in cube $[-1,1]^n$, then
$
\phi(p, \nu)=\frac{2+4p}{5+p},
$
and \eqref{eq:twosided} reduces to
$$
{\mathbb P}\left(1-\delta < \frac{||{\bf x}||_p}{(n\mu_p)^{1/p}} < 1+\delta \right) \geq 1-2\exp\left[-n \cdot \left(\frac{2+4p}{5+p}-\epsilon \right)\delta^2\right].
$$
\item[(c)] If ${\bf x}$ follows standard normal distribution, then
$$
\phi(p, \nu)=\frac{p^2}{2}\left(\frac{\sqrt{\pi}\Gamma\left(\frac{1}{2}+p\right)}{{\color{black}\left(\Gamma\left(\frac{1+p}{2}\right)\right)^2}}-1\right)^{-1},
$$
where $\Gamma(\cdot)$ is the gamma function.
\end{itemize}
\end{remark}

\subsubsection{Case 2: small $p$ and small $\delta$}

Proposition \ref{prop:pto0} investigates the behaviour of $\Lambda^\pm(p,\delta,\nu)$ for small $p$, while Proposition \ref{prop:deltato0} investigates the case of small $\delta$. The following Proposition treats the case when both $p$ and $\delta$ are small.  


\begin{proposition}
Suppose that ${\bf x}$ follows the product distribution $\nu^n$ satisfying Assumptions \ref{assume:a}, \ref{assume:b}, and \ref{assume:c}.
Then
\begin{equation}\label{eq:expansion}
\lim\limits_{p \to 0+}\left(\lim\limits_{\delta \to 0+}\frac{\Lambda^\pm(p,\delta,\nu)}{\delta^2}\right) = \lim\limits_{p \to 0+} \phi(p,\nu)= \lim\limits_{\delta \to 0+}\frac{f^\pm(\delta,\nu)}{\delta^2} = \frac{1}{2\cdot Var[\log |x_i|]},
\end{equation}
where $f^\pm(\delta,\nu)$ and $\phi(p,\nu)$ are defined in \eqref{eq:fdeltadef} and \eqref{eq:phidef}, respectively.
\end{proposition}
\begin{proof}
The proof of the last equality in \eqref{eq:expansion} {\color{black} is similar to} the proof of Proposition \ref{prop:deltato0}:
$$
f^\pm(\delta,\nu) = \sup\limits_{y \geq 0} (\pm y(\log(1\pm\delta)+E[\log |x_i|])-\log E[|x_i|^{\pm y}]) = \sup\limits_{y \geq 0} (ay-h(y))
$$
where $a=\pm(\log(1\pm\delta)+E[\log |x_i|])$ and $h(y):=\log E[|x_i|^{\pm y}]$ is a convex function. Note that $a\geq a_0$, where $a_0:=\pm E[\log |x_i|]$. Because $ay-h(y)\leq 0$ for $y\leq 0$, $f^\pm(\delta,\nu)=g(a)$, where $g(a):=\sup\limits_{y \in {\mathbb R}} (ay-h(y))$ is the Legendre transform of $h$. We then calculate 
$$
h'(y) = \frac{E[\pm \log|x_i|\cdot|x_i|^{\pm y}]}{E[|x_i|^{\pm y}]},
$$
$$
h''(y) = \frac{1}{E[|x_i|^{\pm y}]^2}\left(E[(\log|x_i|)^2|x_i|^{\pm y}]E[|x_i|^{\pm y}]-E[\pm \log|x_i||x_i|^{\pm y}]^2\right),
$$
$h(0)=0$, $h'(0)=\pm E[\log|x_i|]$, $h''(0)=Var[\log|x_i|]$, $g(a_0)=0$, $g'(a_0)=g'(h'(0))=0$,
$$
g''(a_0) = \frac{1}{h''(g'(a_0))} = \frac{1}{h''(0)}=\frac{1}{Var[\log|x_i|]},
$$ 
$$
\lim\limits_{a \to a_0}\frac{g(a)}{(a-a_0)^2} = \frac{1}{2}g''(a_0) = \frac{1}{2 Var[\log|x_i|]},
$$ 
and
$$
\lim\limits_{\delta \to 0+}\frac{f^\pm(\delta,\nu)}{\delta^2} = 
\lim\limits_{\delta \to 0+}\frac{g(\pm\log(1\pm\delta)+a_0)}{\log^2(1\pm\delta)} \cdot 
\lim\limits_{\delta \to 0+}\frac{\log^2(1\pm\delta)}{\delta^2} =  \frac{1}{2\cdot Var[\log |x_i|]}
$$

Because the first two terms in \eqref{eq:expansion} are equal by {\color{black} Proposition \ref{prop:deltato0}}, it remains to prove that
$$
\lim\limits_{p \to 0+} \phi(p,\nu)= \frac{1}{2\cdot Var[\log |x_i|]}.
$$

We claim that $g(p)=\frac{x^p-p\log x-1}{p^2}$ is a non-decreasing function of $p$ on $(0,\infty)$ for any fixed $x\geq 1$ and non-increasing for fixed $0<x\leq 1$. Indeed, $g'(p)=h(p)/p^3$, where $h(p)=2-2x^p+p(1+x^p)\log(x)$, $h'(p)=\log(x)+x^p\log(x)(p\log(x)-1)$, and $h''(p)=px^p\log(x)^3$. Hence, $h(p)$ is convex if $x\geq 1$ and concave if $0<x\leq 1$. Because $h(0)=h'(0)=0$, this implies that for all $p\geq 0$ we have $h(p)\geq 0$ if $x\geq 1$ and $h(p)\leq 0$ if $0<x\leq 1$, and the claim follows. 

The claim implies that for all $p\in(0,1]$, $g(p)\leq g(1)=x-\log(x)-1<x$ if $x\geq 1$, and $g(p)\leq \lim\limits_{p\to 0+}g(p)=\frac{\log^2(x)}{2}\leq \log^2(x)$ if $0<x\leq 1$. Hence, $g(p)\leq \max\left\{x,\log^2(x)\right\}$ for all $x$. Because $E[\max\left\{|x_i|,\log^2(|x_i|)\right\}]<\infty$ by Assumptions \ref{assume:b} and \ref{assume:c}, this implies that 
$$
\lim\limits_{p\to 0+}E\left[\frac{|x_i|^p-p\log|x_i|-1}{p^2}\right]=E\left[\frac{\log^2|x_i|}{2}\right]
$$ 
by the dominated convergence theorem. For the same reason,
$$
\lim\limits_{p\to 0+}E\left[\frac{|x_i|^{2p}-2p\log|x_i|-1}{(2p)^2}\right]=E\left[\frac{\log^2|x_i|}{2}\right].
$$ 
{\color{black} The last two limits imply the second-order expansions
\[
\mathbb E[|x_i|^p]=1+p \mathbb E[\log |x_i|] +\frac{p^2}{2} \mathbb E[(\log |x_i|)^2]+o(p^2), 
\]
and
\[
\mathbb E[|x_i|^{2p}]=1+2p \mathbb E[\log |x_i|] +2p^2 \mathbb E[(\log |x_i|)^2]+o(p^2).
\]}
Therefore,
\[
\begin{array}{ll}
Var[|x_i|^p]& =
E[|x_i|^{2p}]-E[|x_i|^p]^2 \\ 
& =(1+2pE[\log|x_i|]+2p^2E[\log^2|x_i|])\\
& -(1+pE[\log|x_i|]+p^2E[\log^2|x_i|]/2)^2 +o(p^2)\\ 
& =  p^2 Var[\log |x_i|] + o(p^2),
\end{array}
\]
and
\[
\lim\limits_{p \to 0+} \phi(p, \nu)= \lim\limits_{p \to 0+} \frac{p^2}{2}\frac{E[|x_i|^p]^2}{Var[|x_i|^p]} = \frac{1}{2\cdot Var[\log |x_i|]}.
\]
\end{proof}

Examples of the limiting values of $\phi(p,\nu)$ for $p\rightarrow 0$ are provided below.

\begin{itemize} 
\item[(a)] If ${\bf x}$ is uniformly distributed in cube $[-1,1]^n$, then 
$
Var[\log |x_i|] = 1,
$
and 
$$
\lim\limits_{p \to 0+} \phi(p,\nu)= \lim\limits_{p \to 0+} \left(\frac{1}{2}+p\right) = \frac{1}{2\cdot Var[\log |x_i|]} = \frac{1}{2}.
$$
\item[(b)] If ${\bf x}={\bf y}-{\bf z}$, where ${\bf y}$ and ${\bf z}$ are independent and uniformly distributed in cube $[-1,1]^n$, then
$
Var[\log |x_i|] = \frac{5}{4},
$
and 
$$
\lim\limits_{p \to 0+} \phi(p,\nu)= \lim\limits_{p \to 0+} \left(\frac{2+4p}{5+p}\right) = \frac{1}{2\cdot Var[\log |x_i|]} = \frac{2}{5}.
$$
\item[(c)] If ${\bf x}$ follows standard normal distribution, then
$
Var[\log |x_i|] = \frac{\pi^2}{8},
$
and
$$
\lim\limits_{p \to 0+} \phi(p,\nu)= \frac{1}{2\cdot Var[\log |x_i|]} = \frac{4}{\pi^2}.
$$
\end{itemize}

In (a) and (b), we see that
$
\inf\limits_{p>0} \phi(p,\nu) > 0.
$
When this inequality holds, the following uniform in $p$ version of Proposition 
\ref{prop:cherndelta} holds true.

\begin{proposition}\label{prop:pdepla0}
Suppose that ${\bf x}$ follows the product distribution $\nu^n$ satisfying Assumptions \ref{assume:a} and \ref{assume:b}, and let
$$
C^*(\nu) := \inf\limits_{p>0} \phi(p,\nu) > 0.
$$
Then for any $p\in(0,p_0]$, and any $C<C^*{\color{black}(\nu)}$, there is a constant $\delta_0=\delta_0(p,C,\nu)>0$, such that inequality 
\begin{equation}\label{eq:twosided2}
{\mathbb P}\left(1-\delta < \frac{||{\bf x}||_p}{(n\mu_p)^{1/p}} < 1+\delta \right) \geq 1-2\exp[-nC\delta^2]
\end{equation}
holds for all $0<\delta \leq \delta_0$ and all $n\geq 1$.
\end{proposition}
For example, for uniform distribution in a centred cube, $C^*=0.5$, and  inequality
\begin{equation}\label{eq:exp_bound_uniform}
{\mathbb P}\left(1-\delta < \frac{||{\bf x}||_p}{(n\mu_p)^{1/p}} < 1+\delta \right) \geq 1-2\exp[-0.499 \cdot n\delta^2]
\end{equation}
holds for all $p$ and all {\color{black} $0<\delta<\delta_0(p)$}. Note that the bound does not depend on $p$. {\color{black} We remark, however, that the constant $\delta_0$ depends on $p$.} 

\section{Discussion}\label{sec:discussion} 

\subsection{The apparent controversy of fractional quasi p-norms concentration} In the previous section we have shown that, for a large family of distributions, including uniform distributions in $[0,1]^n$ fractional quasi p-norms concentrate. The concentration occurs for all $p>0$ and is at least exponential in $n$. In particular, for $\bf x$ sampled from the latter uniform distribution, the following holds (see Theorem \ref{prop:unip})
\[
{\mathbb P}\left(1-\delta < \frac{||{\bf x}||_p}{(n\mu_p)^{1/p}} < 1+\delta \right) \geq 1-2\exp[-n \cdot f^*(\delta)], 
\quad 
\forall n \geq 1, \ p>0, \ \delta\in(0,1),
\]
where we write $f^*(\delta)$ in place of $f^*(\delta,\nu)$ because the distribution $\nu$ is fixed. 
Therefore, if ${\bf x}_1$ and ${\bf x}_2$ are sampled independently from the same uniform distribution, then
\begin{equation}\label{eq:rel_contrast_cube}
{\mathbb P}\left(\frac{|||{\bf x}_1||_p-||{\bf x}_2||_p|}{(n\mu_p)^{1/p}}<\delta \right)\geq 1 - 4\exp\left[-n \cdot f^*\left(\frac{\delta}{2}\right)\right]
\end{equation}
and, if we are interested in relative contrast, then
\begin{equation}\label{eq:rel_contrast_cube:2}
{\mathbb P}\left(\frac{|||{\bf x}_1||_p-||{\bf x}_2||_p|}{||{\bf x}_1||_p}<\delta \right)\geq 1 - 4\exp\left[-n \cdot f^*\left(\frac{\delta}{2+\delta}\right)\right].
\end{equation}
To see this, consider events
\[
E_1: \ 1-\delta < \frac{||{\bf x}_1||_p}{(n\mu_p)^{1/p}} < 1+\delta, \ E_2: \ 1-\delta < \frac{||{\bf x}_2||_p}{(n\mu_p)^{1/p}} < 1+\delta,
\]
\[
E_3: \ \left|\frac{||{\bf x}_1||_p}{(n\mu_p)^{1/p}} - \frac{||{\bf x}_2||_p}{(n\mu_p)^{1/p}}\right|<2\delta, \ E_4: \ \frac{|||{\bf x}_1||_p-||{\bf x}_2||_p|}{||{\bf x}_1||_p}<\frac{2\delta}{1-\delta}  
\]
It is clear that the joint event $E_1 \& E_2$ implies both $E_3$ and $E_4$, and that 
{\color{black} $\mathbb{P}(E_1\& E_2) \geq 1-\mathbb{P}(\mbox{not} \ E_1) - \mathbb{P}(\mbox{not} \ E_2)$.} 
Then $\mathbb{P}(E_3)\geq \mathbb{P}(E_1\& E_2)$ and $\mathbb{P}(E_4)\geq \mathbb{P}(E_1\& E_2)$. Hence bounds (\ref{eq:rel_contrast_cube}), (\ref{eq:rel_contrast_cube:2}) follow after elementary substitutions. 

At the same time, as has been shown in \cite{aggarwal2001surprising}, 
if ${\bf x}_1, {\bf x}_2$  are drawn independently from the equidistribution in $[0,1]^n$ then the following holds for all $p>0$:
\begin{equation}\label{eq:aggr}
\lim_{n\rightarrow\infty} E\left[\frac{|\|{\bf x}_1\|_p-\|{\bf x}_2\|_p|}{n^{1/p-1/2}}\right]=G \frac{1}{(p+1)^{1/p}} \frac{1}{(2p+1)^{1/2}},
\end{equation}
where $G$ is a positive constant, independent of $p$. 

Although statements, (\ref{eq:rel_contrast_cube}) and (\ref{eq:aggr}) may appear contradictory at the  first glance, a deeper look allows to resolve the tension between the two. Indeed (\ref{eq:aggr}) concerns expectation of $|\|{\bf x}_1\|_p-\|{\bf x}_2\|_p|$, and (\ref{eq:rel_contrast_cube}) estimates the probability of $|\|{\bf x}_1\|_p-\|{\bf x}_2\|_p|<\delta$ ({\color{black} or, equivalently, the probabilities of the complemented events $|\|{\bf x}_1\|_p-\|{\bf x}_2\|_p|\geq \delta$}). The two measures are certainly related, e.g. through the tail integral identity for the expectation:
\[
E[|\|{\bf x}_1\|_p-\|{\bf x}_2\|_p|]=\int_0^\infty \mathbb{P}(|\|{\bf x}_1\|_p-\|{\bf x}_2\|_p|\geq t)dt.
\]
They are, however, evidently not equivalent. Moreover, the normalization constants in (\ref{eq:rel_contrast_cube}) and (\ref{eq:aggr}) are also different (with an additional scaling term $n^{-1/2}$ in (\ref{eq:aggr})). 

Although estimate (\ref{eq:aggr}) may suggest a certain plausibility of varying $p$ to increase discrimination capabilities, our current work shows that varying the values of $p$ in fractional quasi-p-norms does not help to avoid the exponential concentration of $\ell^p$ quasi-norms between samples from $[0,1]^n$ (and for many other distributions) in high dimension.  To illustrate the point more clearly, we computed empirical frequencies of events $\frac{|||{\bf x}_1||_p-||{\bf x}_2||_p|}{(n\mu_p)^{1/p}}<\delta$ for ${\bf x}_1$, ${\bf x}_2$ sampled from $[0,1]^n$. The results are shown in Fig. \ref{fig:uniform_contrast}.
\begin{figure}
\centering
\includegraphics[width=270pt]{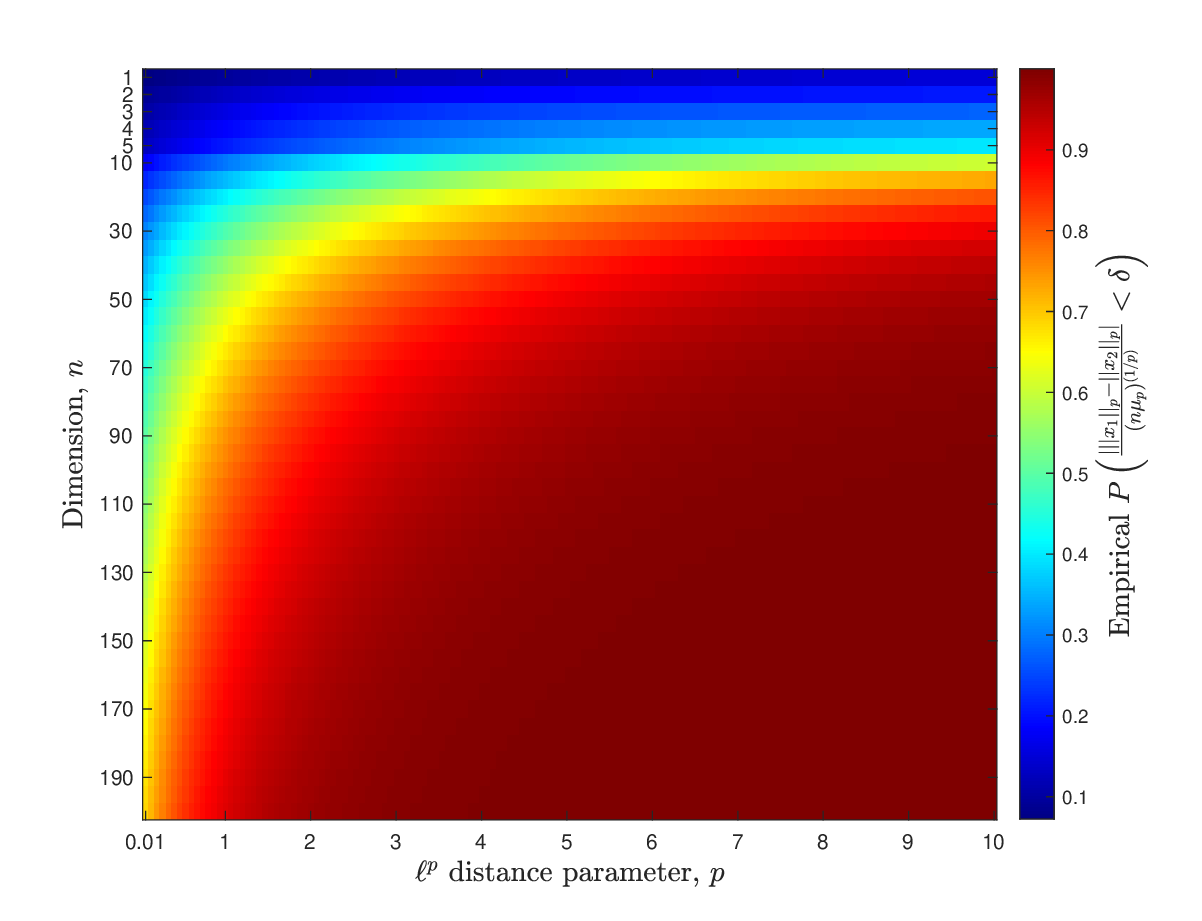}
\includegraphics[width=270pt]{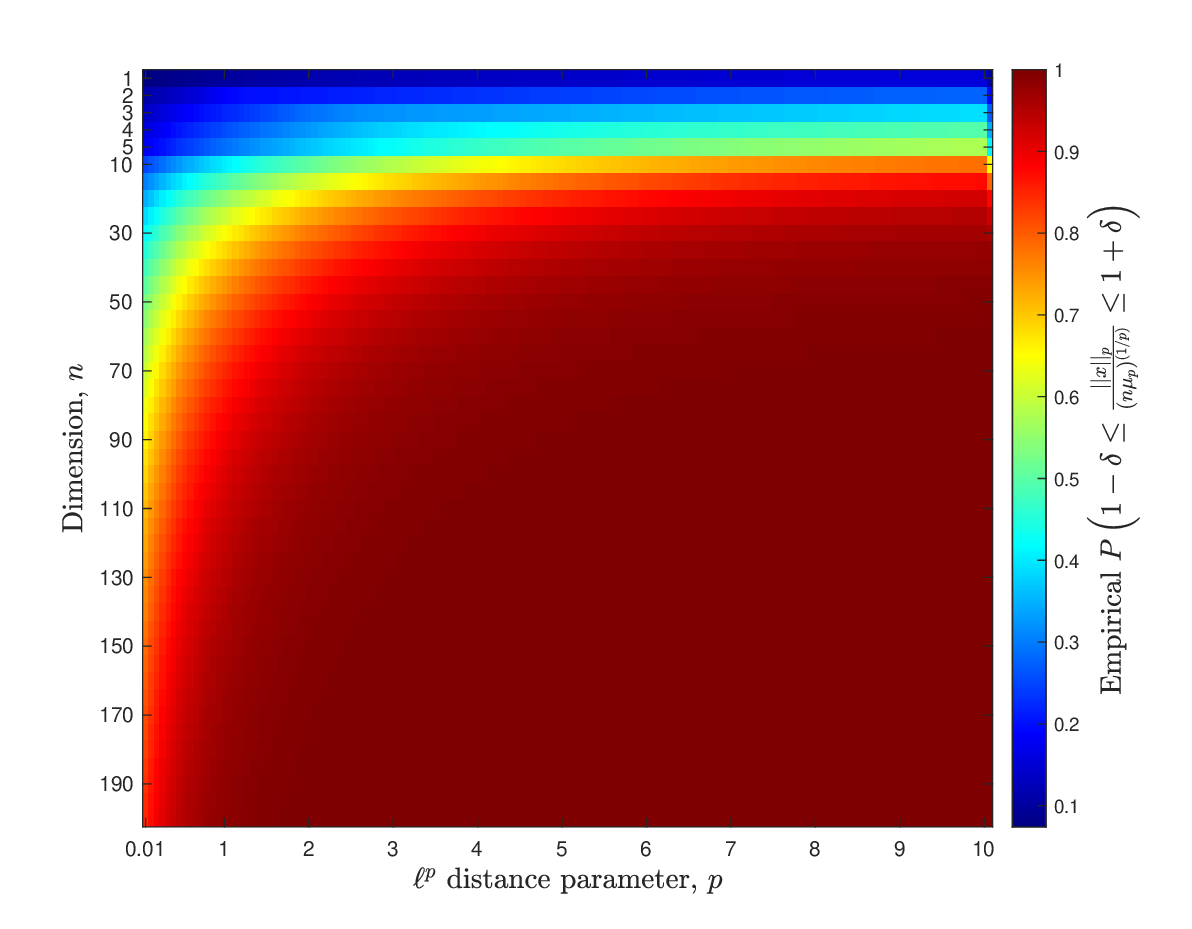}
\caption{{\it Top panel}: Empirical probabilities ${\mathbb P}\left(\frac{|||{\bf x}_1||_p-||{\bf x}_2||_p|}{(n\mu_p)^{1/p}}<\delta \right)$ computed for ${\bf x}_1, \ {\bf x}_2$ sampled from $[0,1]^n$, $\delta = 0.1$; the sample size was equal to $5\times 10^5$. {\it Bottom panel}: Empirical probabilities ${\mathbb P}\left(1-\delta\leq\frac{||{\bf x}||_p}{(n\mu_p)^{1/p}} \leq 1+\delta \right)$ as functions of $p$ and $n$ for $\delta=0.1$, computed for vectors sampled from the uniform distribution  in $[0,1]^n$. The size of data samples used to compute these probabilities was equal to $5\times 10^5$.}\label{fig:uniform_contrast}
\end{figure}
According to the figure, we observe rapid concentration of $\frac{|||{\bf x}_1||_p-||{\bf x}_2||_p|}{(n\mu_p)^{1/p}}$ for all $p$, as expected.  

\subsection{Data encoding and pre-processing schemes may have an impact on $\ell^p$ quasi-norm concentration} Whilst confirming uniform $\ell^p$ exponential concentration for all $p>0$ across a broad range of distributions, we also identified conditions and distributions when concentration can be mitigated by an appropriate choice of $p$. Examples of these distributions are provided in Section \ref{sec:non_uniform}  (see also Proposition \ref{thm:concentration_breaks} and Theorem \ref{thm:concentration_breaks:general}). A signature of these distributions is $\mathbb{P}(x_i=0)>0$. As we showed in Section \ref{sec:uniform} 
{\color{black} (see Proposition \ref{prop:pto0} and Theorem \ref{prop:unip}),} 
distributions with $\mathbb{P}(x_i=0)=0$ do not have the above "mitigation" property. This highlights the importance of data encoding schemes in the context of concentration. For example, a data encoding scheme in which $x_i$ may take values in $\{0,1\}$ will have radically different concentration properties from the one in which $x_i\in\{\epsilon,1\}$, $\epsilon\neq 0$. An example is the so-called "dummy encoding" whereby categorical data are replaced by $0$s and $1$s.

Another important consequence of our work is the practice of working with missing data. A popular way to handle missing data is to replace missed attributes with their averages. This is often followed by centralization and normalization of attributes. This process inevitably results in datasets with zeros filling in all missing attributes. For such data, our anti-concentration results apply (Theorem \ref{thm:concentration_breaks:general}). This means that one can inadvertently remove concentrations by choosing quasi $\ell^p$ norms with $p$ sufficiently small. 

\subsection{Tightness of bounds} Many specific quantitative bounds presented in this work are based on the general theorem (Theorem \ref{theorem:p-norm-concentration}) producing asymptotically tight estimates of the rates. One may still wonder how conservative these and related estimates are in non-asymptotic regimes. To see this, we evaluated numerical concentration rates for samples drawn from $[-1,1]^n$ for different values of $p$, and compared these with theoretical bounds presented  in Proposition \ref{prop:pdepla0}. As we can see from these figures, in this example, our theoretical rates are barely distinguishable from the empirical ones computed for small values of $p$ (for $p=0.01$). 

\subsection{Understanding empirical uncertainties around the utility of fractional quasi p-norms}  Extensive empirical exploration of the relevance of fractional quasi $p$-norms in proximity - based classifiers has been carried out in \cite{mirkes2020fractional}. We reproduced a selection of empirical tests from \cite{mirkes2020fractional} here, and a brief summary of these tests are shown in Table \ref{tab:DBs}. Rows corresponding to high-dimensional datasets (that is datasets whose PCA "condition number" dimension is larger than $10$) are marked by $\ast$. According to this table, we see that in some benchmarks the most accurate outcomes correspond to fractional quasi $p$-norms, and in some instances the best accuracy is attained for classical $p$-norms.  However, for high-dimensional datasets ($6$ datasets in total), fractional quasi $p$-norms have been at the top of the scoreboard just for one dataset, Madelon. In the latter case, the number of informative features is $5$, and the rest are redundant or random.  These observations suggest that fractional quasi $p$-norms may not necessarily have an edge over classical $p$-norms for data with large number attributes. {\color{black} Proposition \ref{thm:concentration_breaks} and Theorem \ref{thm:concentration_breaks:general} specify a class of problems when quasi $p$-norms could be useful. High sparsity of the data is an inherent signature such tasks (see Supplementary Materials for illustrations).
}
{\color{black} In absence of sparsity, for a fixed dimension $n$, there may potentially be a room for mitigating the impact of concentration by choosing $p$ sufficiently small (see Fig. \ref{fig:rates_exponential} showing how exponential concentration rates change with $p$).} Nevertheless, this room will inevitably shrink exponentially in  sufficiently high dimensions {\color{black} due to exponentially fast concentration (see Theorem \ref{prop:unip})}. 

\begin{table}[tb]
\caption{Impact of the choice of $p$ in $\ell^p$ functionals ($p=0.01, 0.1, 0.5, 1, 2, 4 ,10, \infty$) on the accuracy of $k$-nearest neighbours classifies, with $k=11$, in various benchmarks. Input attributes in all data sets were standardised (normalised to zero mean and unit variance). The best accuracy for each data set is highlighted in bold, the worst accuracy is italicized and highlighted in red. Rows  marked by stars indicate datasets whose PCA "condition number" dimension \cite{bac2021scikit} is larger than $10$ (see \cite{mirkes2020fractional} for  details of the derivations).}
\centering
\tiny
\begin{tabular}{lcrrrrrrrrr}
\textbf{Source}&\textbf{\#Attr.}&\textbf{Cases}&\textbf{p=0.01}&\textbf{p=0.1}&\textbf{p=0.5}&\textbf{p=1}&\textbf{p=2}&\textbf{p=4}&\textbf{p=10}&\textbf{p=$\infty$}\\
\hline
Banknote authentication \cite{banknoteDB} & 4 & 1372 & \underline{\color{black}\textit{0.984}} & 0.993 & \underline{\textbf{0.999}} & 0.996 & \underline{\textbf{0.999}} & \underline{\textbf{0.999}} & \underline{0.999} & 0.998 \\\hline
Blood cite{bloodDB} & 4 & 748 & 0.774 & \underline{\color{black} \textit{0.773}} & 0.783 & 0.785 & 0.781 & \underline{\textbf{0.786}} & 0.785 & 0.782 \\\hline
Brest cancer \cite{BreastDB} & 30 & 569 & \underline{\color{black}\textit{0.917}} & 0.953 & 0.960 & 0.967 & \underline{\textbf{0.970}} & 0.961 & 0.949 & 0.933 \\\hline
* Climate Model Simulation &&&&&&&&\\
Crashes \cite{ClimateDB} & 18 & 540 & \underline{\color{black}\textit{0.915}} & \underline{\color{black}\textit{0.915}} & \underline{\textbf{0.917}} & \underline{\textbf{0.917}} & \underline{\textbf{0.917}} & \underline{\color{black}\textit{0.915}} & \underline{\color{black}\textit{0.915}} & \underline{\textbf{0.917}} \\\hline
* Connectionist Bench\\
(Sonar) \cite{SonarDB} & 60 & 208 & \underline{\color{black}\textit{0.702}} & 0.774 & 0.784 & \underline{\textbf{0.788}} & 0.760 & 0.740 & 0.721 & 0.712 \\\hline
Cryotherapy \cite{CryotherapyDB} & 6 & 90 & \underline{\color{black}\textit{0.822}} & 0.833 & 0.889 & \underline{\textbf{0.911}} & 0.867 & 0.833 & 0.833 & 0.833 \\\hline
Diabetic Retinopathy\\
Debrecen \cite{DiabeticDB} & 19 & 1.151 & 0.635 & 0.636 & 0.641 & \underline{\textbf{0.648}} & 0.646 & 0.620 & 0.604 & \underline{\color{black}\textit{0.599}} \\\hline
Digital Colposcopies 1 \cite{ColposcopiesDB} & 62 & 287 & 0.662 & 0.700 & \underline{\textbf{0.704}} & 0.700 & 0.662 & \underline{\color{black}\textit{0.652}} & 0.662 & 0.659 \\\hline
Digital Colposcopies 2 \cite{ColposcopiesDB} & 62 & 287 & 0.669 & \underline{\textbf{0.711}} & 0.690 & 0.679 & \underline{\textbf{0.711}} & 0.669 & \underline{ \color{black}\textit{0.648}} & 0.659 \\\hline
Digital Colposcopies 3 \cite{ColposcopiesDB} & 62 & 287 & 0.669 & 0.679 & \underline{\textbf{0.693}} & 0.662 & 0.662 & 0.641 & 0.627 & \underline{\color{black}\textit{0.596}} \\\hline
Digital Colposcopies 4 \cite{ColposcopiesDB} & 62 & 287 & \underline{\color{black}\textit{0.683}} & 0.721 & \underline{\textbf{0.739}} & 0.728 & 0.735 & 0.721 & 0.711 & 0.725 \\\hline
Digital Colposcopies 5 \cite{ColposcopiesDB} & 62 & 287 & \underline{\color{black}\textit{0.801}} & 0.833 & 0.833 & \underline{\textbf{0.843}} & 0.833 & 0.836 & 0.840 & 0.836 \\\hline
Digital Colposcopies 6 \cite{ColposcopiesDB} & 62 & 287 & 0.631 & 0.669 & 0.669 & \underline{\textbf{0.686}} & 0.659 & 0.666 & 0.655 & \underline{\color{black}\textit{0.620}} \\\hline
Digital Colposcopies 7 \cite{ColposcopiesDB} & 62 & 287 & \underline{\color{black}\textit{0.742}} & 0.767 & \underline{\textbf{0.780}} & 0.777 & 0.767 & 0.770 & 0.770 & 0.760 \\\hline
EEG Eye State \cite{EEGDB} & 14 & 14980 & \underline{\color{black}\textit{0.698}} & 0.776 & \underline{\textbf{0.869}} & 0.841 & 0.833 & 0.832 & 0.827 & 0.828 \\\hline
First-order theorem proving \cite{TheoremDB} & 51 & 6118 & 0.829 & 0.829 & 0.831 & \underline{\textbf{0.833}} & 0.830 & 0.827 & \underline{\color{black}\textit{0.824}} & \underline{\color{black}\textit{0.824}} \\\cline{2-11}
(6 tasks, one row per task) & 51 & 6118 & 0.922 & \underline{\textbf{0.923}} & 0.921 & 0.922 & 0.920 & \underline{\color{black}\textit{0.919}} & 0.919 & 0.919 \\\cline{2-11}
 & 51 & 6118 & 0.877 & 0.878 & \underline{\textbf{0.882}} & 0.880 & 0.880 & \underline{\color{black}\textit{0.876}} & 0.877 & 0.878 \\\cline{2-11}
   & 51 & 6118 & 0.898 & 0.900 & 0.899 & \underline{\textbf{0.901}} & 0.900 & 0.899 & \underline{\color{black}\textit{0.898}} & \underline{\color{black}\textit{0.898}} \\\cline{2-11}
  & 51 & 6118 & 0.903 & 0.905 & 0.905 & \underline{\textbf{0.906}} & 0.905 & 0.904 & 0.902 & \underline{\color{black}\textit{0.901}} \\\cline{2-11}
  & 51 & 6118 & 0.789 & 0.801 & \underline{\textbf{0.804}} & 0.803 & 0.801 & 0.792 & \underline{\color{black}\textit{0.788}} & \underline{\color{black}\textit{0.788}} \\\hline
* Gisette \cite{GisetteDB} & 5000 & 7000 & 0.779 & 0.866 & 0.947 & \underline{\textbf{0.963}} & 0.961 & 0.714 & 0.619 & \underline{\color{black}\textit{0.591}} \\\hline
HTRU2 \cite{HTRU2BD, HTRU2BD2} & 8 & 17 & \underline{\color{black}\textit{0.975}} & 0.977 & 0.978 & \underline{\textbf{0.979}} & 0.978 & 0.978 & \underline{\textbf{0.979}} & 0.978 \\\hline
ILPD (Indian Liver &&&&&&&&\\
Patient Dataset) \cite{ILDPDB} & 10 & 579 & 0.693 & 0.689 & 0.689 & \underline{\textbf{0.694}} & \underline{\color{black}\textit{0.656}} & 0.660 & 0.668 & 0.680 \\\hline
Immunotherapy \cite{ImmunotherapyDB} & 7 & 90 & \underline{\color{black}\textit{0.800}} & \underline{\color{black}\textit{0.800}} & 0.811 & \underline{\textbf{0.822}} & 0.811 & \underline{\color{black}\textit{0.800}} & 0.811 & \underline{\color{black}\textit{0.800}} \\\hline
Ionosphere \cite{IonosphereDB} & 34 & 351 & 0.852 & 0.903 & \underline{\textbf{0.912}} & 0.872 & 0.829 & \underline{\color{black}\textit{0.818}} & 0.855 & 0.846 \\\hline
* Madelon \cite{MadelonDB} & 500 & 2600 & 0.510 & 0.536 & \underline{\textbf{0.626}} & 0.600 & 0.564 & 0.523 & 0.512 & \underline{\color{black}\textit{0.491}} \\\hline
MAGIC Gamma \\
Telescope \cite{TelescopeDB} & 10 & 19020 & \underline{\color{black}\textit{0.704}} & 0.795 & 0.839 & \underline{\textbf{0.844}} & 0.842 & 0.840 & 0.836 & 0.835 \\\hline
MiniBooNE particle &&&&&&&&\\
identification \cite{MiniBooNEDB} & 50 & 130064 & \underline{\color{black}\textit{0.843}} & \underline{\textbf{0.923}} & 0.913 & 0.900 & 0.892 & 0.890 & 0.889 & 0.888 \\\hline
Musk 1 \cite{MuskDB} & 166 & 467 & 0.784 & 0.788 & 0.821 & \underline{\textbf{0.853}} & 0.840 & 0.824 & 0.769 & \underline{\color{black}\textit{0.739}} \\\hline
Musk 2 \cite{MuskDB} & 166 & 6598 & 0.945 & 0.946 & 0.955 & 0.957 & \underline{\textbf{0.962}} & 0.962 & 0.949 & \underline{\color{black}\textit{0.937}} \\\hline
Planning Relax \cite{RelaxDB} & 10 & 182 & \underline{\textbf{0.703}} & \underline{\textbf{0.703}} & 0.681 & 0.681 & \underline{\color{black}\textit{0.659}} & 0.670 & 0.670 & 0.681 \\\hline
* QSAR biodegradation \cite{QSARDB} & 41 & 1.055 & 0.839 & 0.851 & 0.861 & \underline{\textbf{0.865}} & 0.856 & 0.849 & 0.833 & \underline{\color{black}\textit{0.828}} \\\hline
* SPECT Heart \cite{HeartDB} & 22 & 267 & \underline{\textbf{0.846}} & \underline{\textbf{0.846}} & \underline{\textbf{0.846}} & \underline{\textbf{0.846}} & \underline{\textbf{0.846}} & \underline{\textbf{0.846}} & \underline{\color{black}\textit{0.831}} & 0.843 \\\hline
 SPECTF Heart \cite{HeartDB} & 44 & 267 & 0.772 & 0.764 & 0.779 & \underline{\textbf{0.798}} & 0.768 & 0.738 & \underline{\color{black}\textit{0.712}} & 0.749 \\\hline
Vertebral Column \cite{VertebralDB} & 6 & 310 & 0.777 & \underline{\color{black}\textit{0.768}} & 0.810 & 0.794 & \underline{\textbf{0.823}} & 0.813 & 0.803 & 0.800 \\\hline

\end{tabular}
\label{tab:DBs}
\end{table}

\begin{figure}
\centering
\includegraphics[width=210pt]{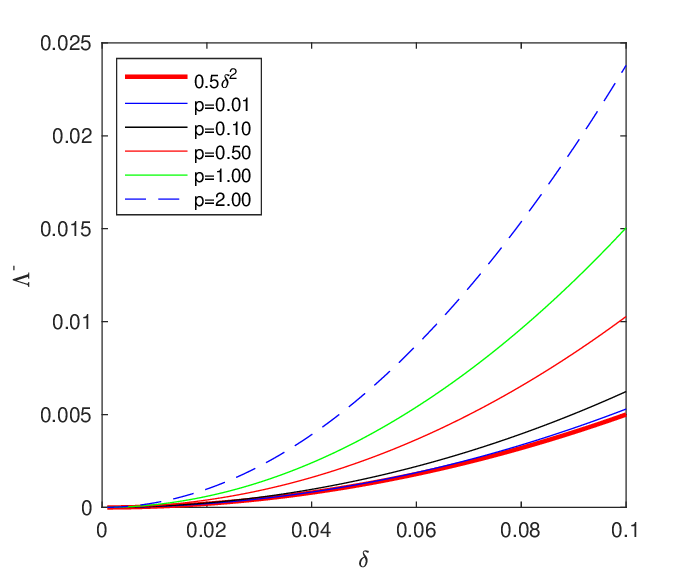}
\includegraphics[width=210pt]{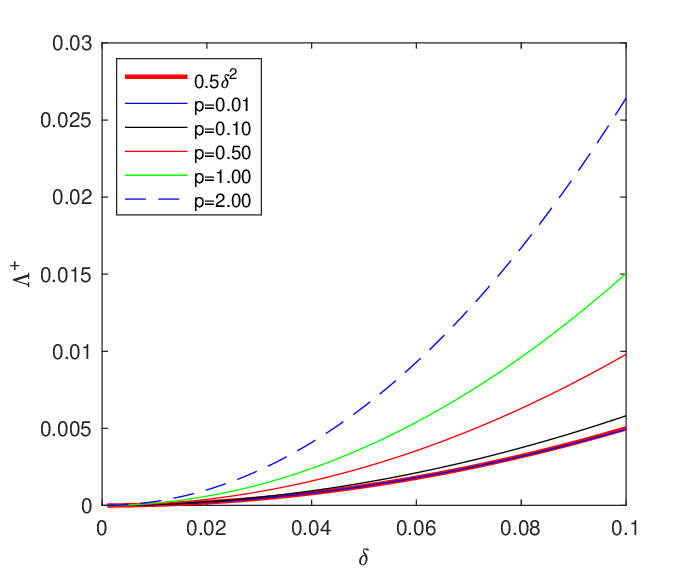}
\caption{Theoretical exponential concentration rates for the uniform distribution in $[-1,1]^n$. Left panel shows $\Lambda^-(p,\delta)$ for $p=0.01, 0.1, 0.5, 1, 2$. Right panel shows $\Lambda^+(p,\delta)$ for $p=0.01, 0.1, 0.5, 1, 2$ (see Theorem \ref{theorem:p-norm-concentration}). Red solid line in both panels shows the bound stemming from Proposition \ref{prop:pdepla0}, and which is independent on the values of $p$ (see  (\ref{eq:exp_bound_uniform})).}\label{fig:rates_exponential}
\end{figure}

\subsection{Ubiquity of anti-concentration} 

One of the main conclusions of this work is that we can remove the concentration around the mean length of random vectors with potentially high dimensions by choosing $p$ sufficiently small, provided that ${\mathbb P}(x_i=0)>0$. This condition fails for all continuous probability distributions, and for all discrete distributions with no atom at $0$. However, if this condition fails for $x_i$, we can make it hold by applying an atrbitrary small perturbation to $x_i$. Specifically, we can replace random variable $x_i$ by random variable $\overline{x}_i(a)$ which is equal to $0$ with probability $a$ and to $x_i$ with probability $1-a$, where $a>0$ is a small parameter. Then, intuitively, $\overline{x}_i(a)$ is a small perturbation of $x_i$ for which anti-concentration theorem is applicable. This intuition can be formalized by observation that $\overline{x}_i(a)$ converges in distribution to $x_i$ as $a\to 0+$. 

An alternative way to formalize the closeness of $\overline{x}_i(a)$ to $x_i$ is to say that the Wasserstein distance between them is small.
Let $\mu,\nu$ be two probability measures on $\Real^n$, with finite $q$-th moments of $\ell^d$, $d\geq 1$ distance: $E_{{\bf x}\sim \mu}[\|{\bf x}\|_d^q]<\infty$, $E_{{\bf y}\sim \nu}[\|{\bf y}\|_d^q]<\infty$. Recall that the Wasserstein $(d,q)$-distance, $q\geq 1$ between $\mu,\nu$ is
\[
W_{d,q}(\mu,\nu):=\inf_{\pi\in\Pi(\mu,\nu)} \ 
\left(E_{({\bf x}, {\bf y})\sim\pi} [\|{\bf x} - {\bf y}\|_d^{q}]\right)^{1/q},
\]
where $\Pi(\mu,\nu)$ is the set of all probability measures $\pi$ on the product space $\Real^n\times\Real^n$ satisfying
\[
\pi(A \times \Real^n)=\mu(A), \ \pi(\Real^n \ \times B)=\nu(B)
\]
for all measurable subsets $A$ and $B$ of $\Real^n$ (see \cite{Villani_2003} for further details). For all  such $q,d\geq 1$, $W_{d,q}$ defines a metric on the space of probability measures on $\Real^n$ (\cite{Villani_2003}, Theorem 7.3). For any $\varepsilon>0$, we say that probability distributions $\mu$ and $\hat{\mu}$ on ${\mathbb R}^n$ are $\varepsilon$-close if 
\[
W_{1,1}(\mu,\hat{\mu})=\inf_{\gamma\in\Gamma(\mu,\hat{\mu})} \ E_{({\bf x}, {\bf y})\sim\gamma} [\|{\bf x}-{\bf y}\|_1]<\varepsilon.
\]

\begin{corollary}\label{cor:anti} Let $\mathcal{M}_b(n)$ be the class of product measure distributions {with {\color{black} non-identically-$0$} i.i.d. components} whose support is in $[-b,b]^n$, $b>0$. Then for any $\varepsilon>0$, and any $\mu \in \mathcal{M}_b(n)$, there exist uncountably many distributions $\hat{\mu}\in\mathcal{M}_b(n)$ $\varepsilon$-close to $\mu$ such that for ${\bf x}$ following distribution $\hat{\mu}$ the following holds: 

for any $\delta\in(0,1)$ and $\Delta\in(0,1)$ 
there exist 
a $p^\ast(n,\Delta,\delta,\varepsilon,\hat{\mu})$ such that
\[
\mathbb{P}\left(1-\delta \leq \frac{\|{\bf x}\|_p}{(n\mu_p)^{1/p}} \leq  1+ \delta\right) \leq \Delta
\]
for all $p\in(0,p^\ast(n,\Delta,\delta,\varepsilon,\hat{\mu}))$.
\end{corollary}
\begin{proof} Let
\[
\mu=\mu_0 \times \cdots \mu_0
\]
be an element of $\mathcal{M}_b(n)$. Consider $\mu_0$ with its corresponding $\sigma$-algebra $\Sigma$. 
 For any $a\in(0,1)$, 
introduce
\[
\hat{\mu}_0(A)=\left\{\begin{array}{ll} (1-a)\mu_0(A), & 0\notin A\\
 a+(1-a)\mu_0(A), & 0\in A \end{array}\right.
\]
Define
\[
\hat{\mu}=\hat{\mu}_0\times \cdots \times \hat{\mu}_0.
\]
It is clear that $\hat{\mu}\in \mathcal{M}_b(n)$. 

{ Let $\Pi_1(\mu_0,
\hat{\mu}_0)$ be the set of all probability measures $\pi_1$ on the product space $\Real\times\Real$ satisfying
\[
\pi_1(A \times \Real)=\mu_0(A), \ \pi_1(\Real \ \times B)=
\hat{\mu}_0(B)
\]
for all measurable subsets $A$ and $B$ of $\Real$. Given that both $\mu$ and $\hat{\mu}_0$ are product measures, $
\Pi_1(\mu_0,\hat{\mu}_0)\times\cdots\times\Pi_1(\mu_0,\hat{\mu}_0)\subset \Pi(
\mu,\hat{\mu})$. Therefore
\[
W_{1,1}(\mu,\hat{\mu})=\inf_{\gamma\in\Pi(\mu,\hat{\mu})} \ E_{({\bf x}, {\bf y})\sim\gamma} [\|{\bf x}-{\bf y}\|_1] \leq \inf_{\gamma\in \Pi_1(\mu_0,\hat{\mu}_0)\times\cdots\times\Pi_1(\mu_0,\hat{\mu}_0)}
E_{({\bf x},{\bf y})\sim \gamma}[\|{{\bf x}-{\bf y}}\|_{1}
\]
\begin{equation}\label{eq:W_distance}
=\sum_{i=1}^n \inf_{\gamma_1\in \Pi_1(\mu_0,\hat{\mu}_0)} E_{(x_i,y_i)\sim \gamma_1} |x_i-y_i|=n \int_{-b}^{b} |F_0(x)-\hat{F}_0(x)| dx, 
\end{equation}
}
where $F_0$ and $\hat{F}_0$ are cumulative distribution functions for $\mu_0$ and $\hat{\mu}_0$, respectively. The last equality in (\ref{eq:W_distance}) follows from \cite{Villani_2003} (Remark 2.19, p. 77). Therefore,
\[
W_{1,1}(\mu,\hat{\mu})\leq 2 n b a.  
\]
Picking $a<\varepsilon/{2nb}$ ensures that $W_{1,1}(\mu,\hat{\mu})< \varepsilon$. Let $0<a<\varepsilon/{2nb}$ be such that
$a+(1-a)\mu_0(\{0\})$ is an irrational number. 
Since the set of such $a$ is uncountable, uncountably many such distributions exist.

Note that the corresponding distributions $\hat{\mu}$ all satisfy Assumptions  \ref{assume:a}, \ref{assume:d}. The statement now follows from  Theorem \ref{thm:concentration_breaks:general}.
\end{proof}

\section{Conclusion}\label{sec:conclusion}

In this paper we presented an answer to the long-standing question around the feasibility of using fractional $\ell^p$ quasi-norms, with $p\in(0,1)$, to mitigate the challenge of $\ell^p$ quasi-norms concentration in high dimension. {\color{black} Our answer confirms that for a large class of distributions without degeneracies at zero, there is an exponential concentration bound of $\ell^p$ quasi-norms which holds true for all admissible $p>0$.} This bound holds true for distributions considered in \cite{aggarwal2001surprising} and for which a conjecture  was made on the potential utility of fractional quasi p-norms as a means to avoid or alleviate concentrations. 

At the same time, we prove that there are distributions for which  $\ell^p$ quasi-norm concentrations could be effectively reduced and controlled via the choice of $p$. These distributions, however, are different from the ones considered in \cite{aggarwal2001surprising}. We identified that the key determinant impacting on the occurrence or absence of the concentrations is the behavior of the data distribution at zero. Distributions which are not  concentrating at zero will inevitably adhere to uniform exponential concentration bounds. On the other hand, distributions that concentrate at zero may be receptive to concentration controls via  the choice of fractional $p$.  

These results explain why both phenomena, that is the concentration of fractional quasi $p$-norms or the lack of it, may be observed in experiments. The comprehensive analysis we provided in this paper, however, is limited to the standard i.i.d. case. We also focused on the purely theoretical take on the problem and distanced from evaluating the impact of both negative and positive results reported in this  paper for specific clustering, pattern recognition, and classification tasks. Exploring practical implications and generalizing results beyond the i.i.d. setting could be topics for future work.

\bibliographystyle{siamplain}
\bibliography{references}

\renewcommand\thefigure{SM\arabic{figure}}\setcounter{figure}{0}
\renewcommand\thetable{SM\arabic{table}}\setcounter{table}{0}
\renewcommand\thesection{SM\arabic{section}}\setcounter{section}{0}

\section{Supplementary Materials: implications for modern data science}

Theoretical results presented in the main body of the paper bear direct  implications for several core technologies in modern AI and machine learning. Here we provide examples of specific areas where the dichotomy between concentration and anti-concentration identified in Theorems \ref{thm:concentration_breaks:general} and \ref{prop:unip} has practical consequences, present supporting references, and report the results of  relevant numerical experiments. These areas are: data preprocessing practices (including imputation and encoding) and dimensionality reduction through neighbour embeddings such as in UMAP or t-SNE~\cite{McInnes2018UMAP}, \cite{vanderMaaten2008tSNE}.

In Section \ref{sec:preprocessing} of this Supplementary Material we relate theoretical predictions stemming from our paper (namely Theorems \ref{thm:concentration_breaks:general} and \ref{prop:unip}) with surprising and somewhat unexpected consequences of common data preprocessing steps -- the emergence of anti-concentration due to statistically insignificant common data manipulation or due to the choice of data encoding. The predictions are tested on both synthetic data (random samples from the uniform distribution on $[0,1]^{30}$) and on two publicly available datasets: Gene Expression dataset \cite{gene_expression_cancer_rna-seq_401}, and Wisconsin Breast Cancer dataset \cite{BreastDB}.  For all these datasets, we reveal that certain practices may inevitably lead to an anti-concentration phenomenon, as predicted by Theorem \ref{thm:concentration_breaks:general}. Importantly, these practices preserve statistical properties of data, as measured by the standard Kolmogorov-Smirnov test and Wasserstein distances between distributions (see Corollary \ref{cor:anti}). In Section \ref{sec:retrieval} we demonstrate how our theory could be applied to the problem of information retrieval, including in Retrieval Augmented Generation in modern Large Language Models. In Section 
\ref{sec:umap} we present another potential application area -- dimensionality reduction and visualisation of high-dimensional data. In Section \ref{sec:experiment} we present additional experiments on synthetic data illustrating the impact of dense and sparse encoding on concentration and how passing originally non-atomic data through networks with ReLU activations may lead to anti-concentration. 

\subsection{Data preprocessing and concentration}
\label{sec:preprocessing}

In this section we illustrate how common preprocessing steps such as mean imputation, dummy encoding or the way categorical attributes are encoded in a dataset may influence the behaviour of $\ell^{p}$ quasi-norms. We show that the efficiency of  $\ell^p$ quasi-norms for small~$p$ is not a property of the data alone but also could depend on the \emph{data encoding} and \emph{preprocessing}. We present a series of experiments illustrating the sensitivity of these phenomena on both synthetic and real-world datasets. 

In each experiment, we start with a dataset that has no zero atoms and perturb it by introducing zeros with some probability~$P_{\mathrm{gap}}$. The replacement models the impact of popular mean value imputation followed by centralization, rounding-off noise in finite-precision arithmetic, and the encoding of categorical variables. For small values of ~$P_{\mathrm{gap}}$, this process produces a dataset sampled from a distribution located within a small Wasserstein neighbourhood of the original. We then compare the concentration behaviour of $\ell^p$ quasi-norms before and after the perturbation.

\subsubsection{Synthetic data}\label{sec:cube30}
 
In order to check the consequences of mean value imputation in a controlled setting, we generated two datasets, each comprising 500 observations of 30 dimensional vectors. The first dataset contained observations whose attributes are i.i.d. samples from the uniform distribution in $[0,1]$. The second dataset was a modification of the first dataset in which some elements have been replaced by $0$. The probability of replacement, denoted as $P_{\mathrm{gap}}$, varied in the interval $[0.01, 0.10]$. 

\paragraph{Kolmogorov--Smirnov test}
For each value of~$P_{\mathrm{gap}}$, we applied a two-sample Kolmogorov--Smirnov (KS) test to attributes from these two datasets. The minimal $p$-value across all dimensions is reported in Table \ref{tab:cube30}. At $P_{\mathrm{gap}} = 0.01$, the two distributions are statistically indistinguishable ($p = 1$). At $P_{\mathrm{gap}} = 0.1$ the KS test  still does not warrant confident rejection of the null hypothesis that the datasets were sampled from identical distributions.

\paragraph{Observed behaviour} Table \ref{tab:cube30} reports the Wasserstein $W_{2,1}$ distances and KS $p$-values between the original and zero-imputed datasets for $P_{\mathrm{gap}} \in [0.01, 0.10]$. Even for relatively large gap probabilities (e.g., $P_{\mathrm{gap}} = 0.10$), the KS test does not reject at the 5\% level. Nevertheless, the concentration behaviour of $\ell^p$ quasi-norms is visibly affected. Figure \ref{fig:cube30} shows the concentration fraction as a function of~$p$ for various values of gap probabilities: as~$P_{\mathrm{gap}}$ increases, the curves for the imputed data  diverge away from the original, particularly for small~$p$.
 
\begin{table}[htbp]
\centering\footnotesize
\caption{Wasserstein distance and KS $p$-values for Cube~30 ($n = 30$, $M = 500$).}\label{tab:cube30}
\begin{tabular}{@{}ccc@{}}
\toprule
$P_{\mathrm{gap}}$ & Wasserstein dist.\ & KS $p$-value \\
\midrule
0.01 & 0.275 & 1.000 \\
0.02 & 0.493 & 1.000 \\
0.03 & 0.634 & 0.978 \\
0.04 & 0.795 & 0.903 \\
0.05 & 0.886 & 0.770 \\
0.06 & 1.119 & 0.613 \\
0.07 & 1.214 & 0.370 \\
0.08 & 1.358 & 0.329 \\
0.09 & 1.487 & 0.173 \\
0.10 & 1.593 & 0.150 \\
\bottomrule
\end{tabular}
\end{table}
 
\begin{figure}[htbp]
\centering
\begin{minipage}[t]{0.48\textwidth}\centering
\includegraphics[width=\textwidth]{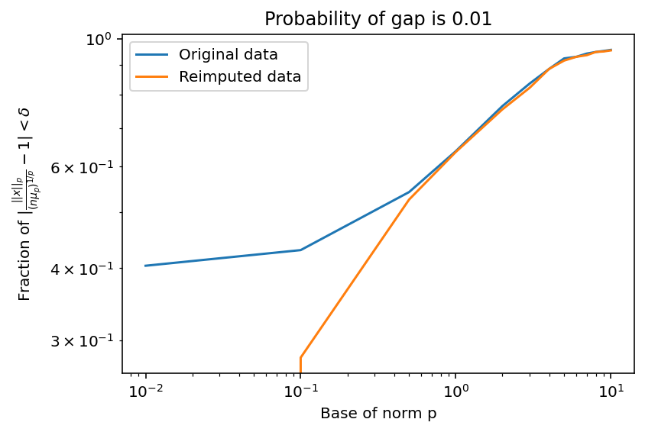}
\end{minipage}\hfill
\begin{minipage}[t]{0.48\textwidth}\centering
\includegraphics[width=\textwidth]{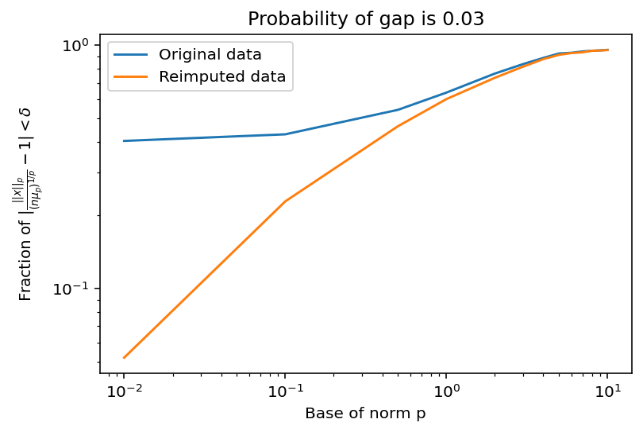}
\end{minipage}\hfill
\begin{minipage}[t]{0.48\textwidth}\centering
\includegraphics[width=\textwidth]{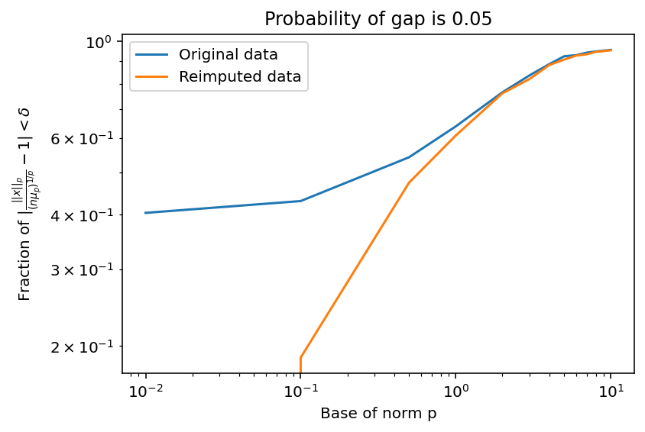}
\end{minipage}\hfill
\begin{minipage}[t]{0.48\textwidth}\centering
\includegraphics[width=\textwidth]{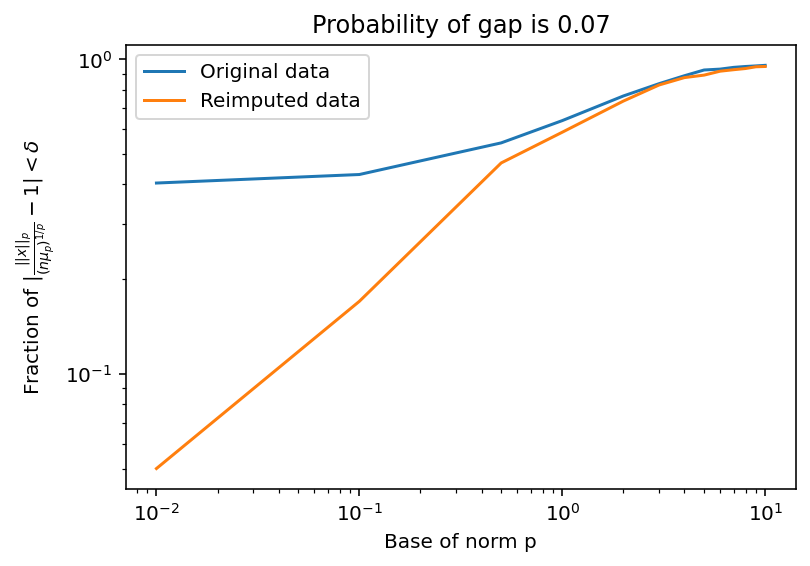}
\end{minipage}\hfill
\begin{minipage}[t]{0.48\textwidth}\centering
\includegraphics[width=\textwidth]{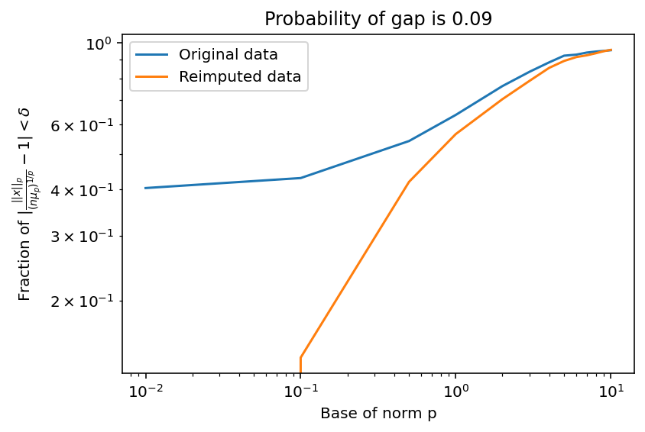}
\end{minipage}\hfill
\begin{minipage}[t]{0.48\textwidth}\centering
\includegraphics[width=\textwidth]{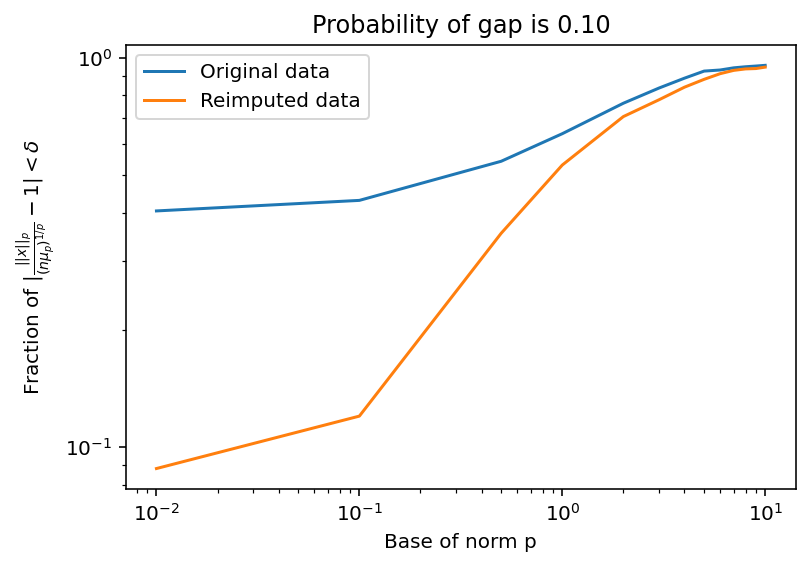}
\end{minipage}
\caption{Cube~30 ($n = 30$, $M = 500$): fraction of observations with $\bigl|\|x\|_p/(n\mu_p)^{1/p} - 1\bigr| < \delta$ as a function of~$p$, for different values of $P_{\mathrm{gap}}$. Original data is shown in blue, and data subjected to zero imputation is shown in orange.}
\label{fig:cube30}
\end{figure}
 
 \subsubsection{Wisconsin Diagnostic Breast Cancer (WDBC) \cite{BreastDB}}\label{sec:wdbc}
 
 We further tested the sensitivity of concentration to small perturbations on the Wisconsin Diagnostic Breast Cancer dataset (30 attributes, 569 observations). The number of attributes and observations in this dataset match those of the synthetic dataset considered earlier. 

In order to simulate the impact of imputation due to the need to fill in missing attributes in medical records, we followed the approach presented in the beginning of the section.  Table \ref{tab:wdbc} reports corresponding Wasserstein distances and KS $p$-values. As observed for the synthetic dataset, even relatively high gap probabilities produce Wasserstein distances of order~$1$, and the KS test $p$-values decrease steadily from~$1.000$ at $P_{\mathrm{gap}} = 0.01$ to~$0.057$ at $P_{\mathrm{gap}} = 0.10$. Figure \ref{fig:wdbc} depicts empirical concentration fractions as functions of $p$ for various values of $P_{\mathrm{gap}}$. For the WDBC data, the divergence between original and imputed concentration behaviour is visible for all values of $P_{\mathrm{gap}}$. Strikingly, the differences extend to sufficiently large values of $p\geq 1$, confirming that the divergence away from nominal concentration behaviour exists in entire intervals of $p$ rather than at $p=0$, as prescribed in Theorem \ref{thm:concentration_breaks:general}. Moreover, the larger the values of $P_{\mathrm{gap}}$ are, the wider the corresponding intervals become.

\begin{table}[htbp]
\centering\footnotesize
\caption{Wasserstein distance and KS $p$-values for the WDBC dataset ($n = 30$, $M = 569$).}\label{tab:wdbc}
\begin{tabular}{@{}ccc@{}}
\toprule
$P_{\mathrm{gap}}$ & Wasserstein dist.\ & KS $p$-value \\
\midrule
0.01 & 0.226 & 1.000 \\
0.02 & 0.391 & 1.000 \\
0.03 & 0.595 & 0.909 \\
0.04 & 0.699 & 0.874 \\
0.05 & 0.833 & 0.544 \\
0.06 & 1.000 & 0.367 \\
0.07 & 1.032 & 0.205 \\
0.08 & 1.197 & 0.205 \\
0.09 & 1.297 & 0.057 \\
0.10 & 1.327 & 0.057 \\
\bottomrule
\end{tabular}
\end{table}
 
\begin{figure}[htbp]
\centering
\begin{minipage}[t]{0.48\textwidth}\centering
\includegraphics[width=\textwidth]{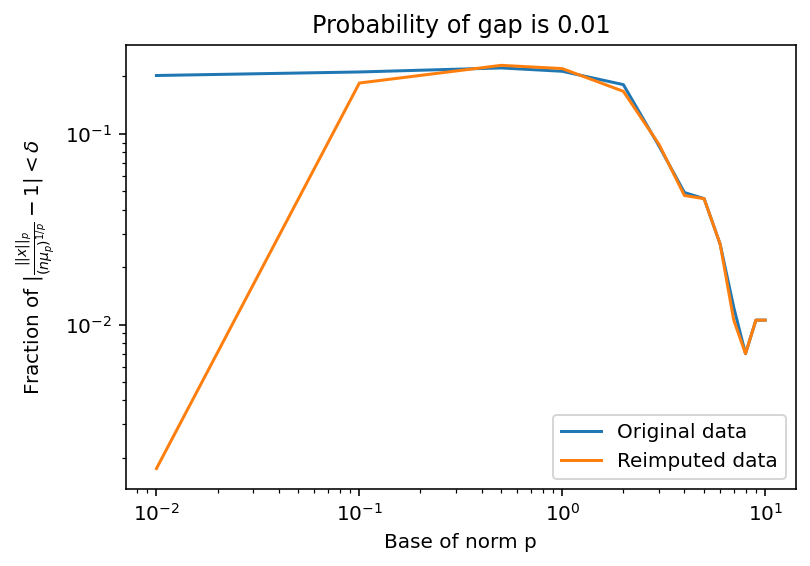}
\end{minipage}\hfill
\begin{minipage}[t]{0.48\textwidth}\centering
\includegraphics[width=\textwidth]{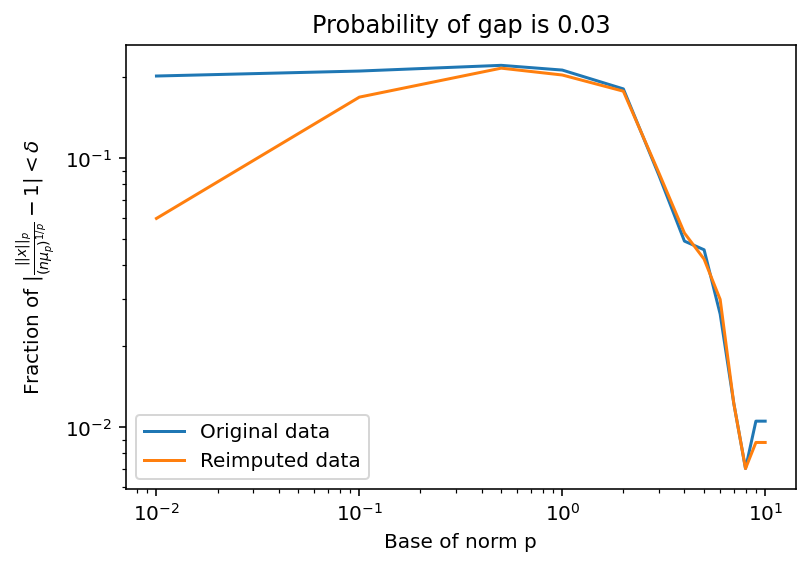}
\end{minipage}\hfill
\begin{minipage}[t]{0.48\textwidth}\centering
\includegraphics[width=\textwidth]{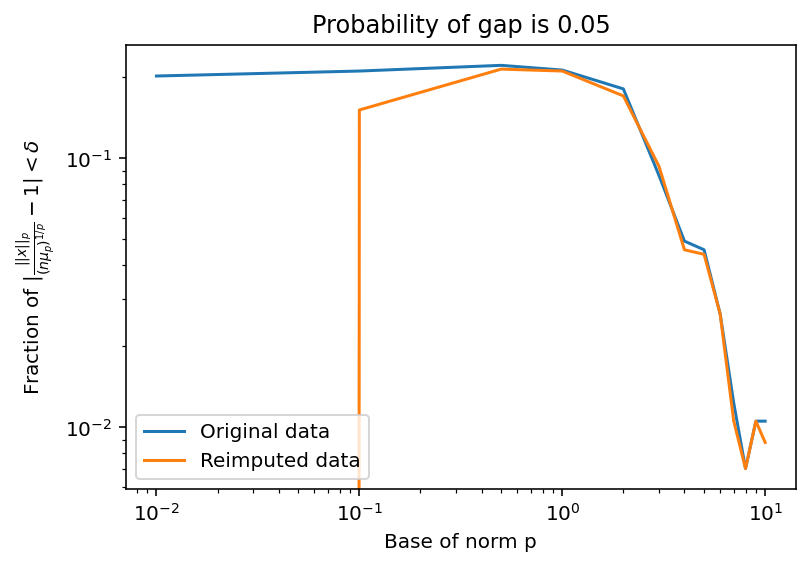}
\end{minipage}\hfill
\begin{minipage}[t]{0.48\textwidth}\centering
\includegraphics[width=\textwidth]{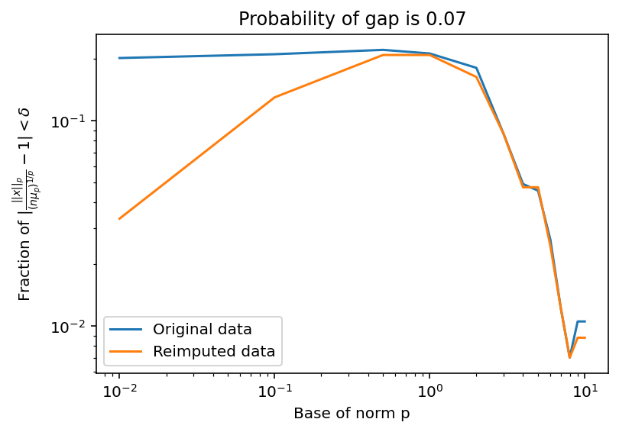}
\end{minipage}\hfill
\begin{minipage}[t]{0.48\textwidth}\centering
\includegraphics[width=\textwidth]{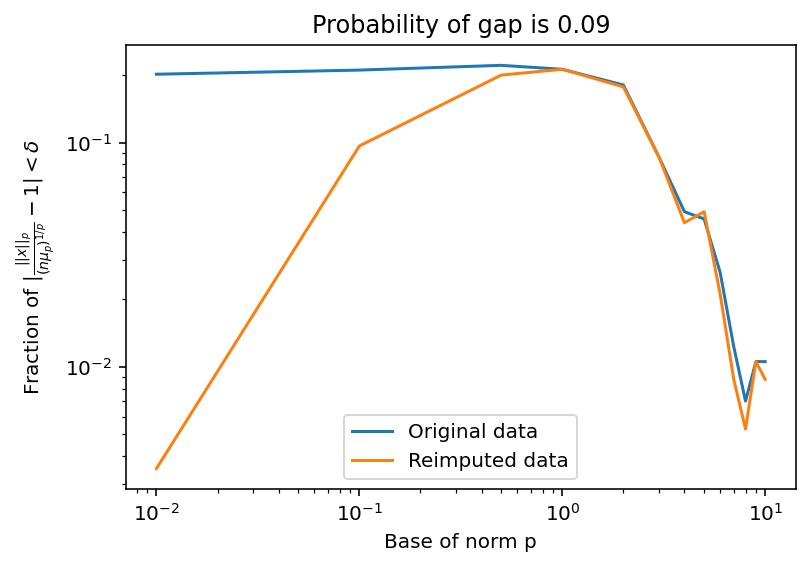}
\end{minipage}\hfill
\begin{minipage}[t]{0.48\textwidth}\centering
\includegraphics[width=\textwidth]{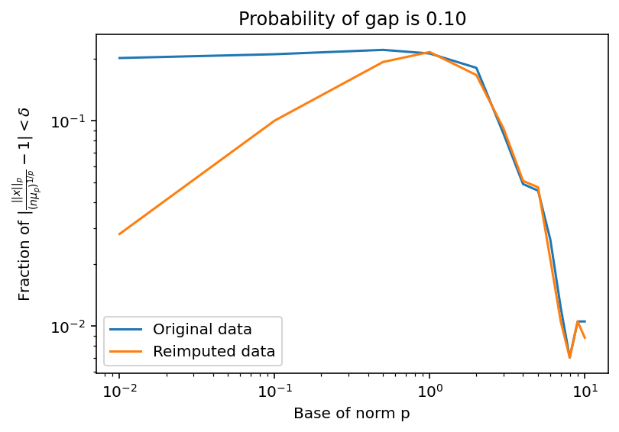}
\end{minipage}
\caption{WDBC dataset ($n = 30$, $M = 569$): fraction of observations with $\bigl|\|x\|_p/(n\mu_p)^{1/p} - 1\bigr| < \delta$ as a function of~$p$, for different values of $P_{\mathrm{gap}}$. Original data is shown in blue, and data subjected to zero imputation is shown in orange.}
\label{fig:wdbc}
\end{figure}
 
We note another interesting artefact present in Figure \ref{fig:wdbc}. For $p$ large, the fraction of observations with $\bigl|\|x\|_p/(n\mu_p)^{1/p} - 1\bigr| < \delta$ sharply drops. This can be explained by extreme sensitivities of quantities $\|x\|_p/(n\mu_p)^{1/p}$ to outliers and other irregularities for $p$ large.

\subsubsection{Gene expression data \cite{gene_expression_cancer_rna-seq_401}}\label{sec:gene_expr}
 
We applied the same Wasserstein neighbourhood perturbation to a gene expression dataset consisting of 801 samples and 20{,}531 attributes. After removing all constant columns (267 attributes with a single unique value), the dataset exhibits considerable variability in the number of distinct values per attribute (Table \ref{tab:gene_uv}). In the dataset comprising records with remaining  20{,}264 attributes, 19{,}849 attributes contain more than 25\% unique values (i.e., more than 200 distinct values out of 801 observations), and only 6{,}223 attributes contain all 801 unique values.
 
\begin{table}[htbp]
\centering\footnotesize
\caption{Distribution of unique values per attribute in the gene expression dataset.}\label{tab:gene_uv}
\begin{tabular}{@{}rc@{\qquad}rc@{\qquad}rc@{}}
\toprule
\# UV & Cols & \# UV & Cols & \# UV & Cols \\
\midrule
1 & 267 & 2 & 30 & 3 & 29 \\
4 & 27 & 5 & 27 & 6 & 19 \\
7 & 23 & 8 & 20 & 9 & 35 \\
10 & 24 & 11 & 23 & 12 & 17 \\
13 & 24 & 14 & 20 & 15 & 22 \\
16 & 14 & 17 & 16 & 18 & 14 \\
19 & 13 & 20 & 17 & 21 & 23 \\
22 & 21 & 23 & 21 & 24 & 11 \\
\bottomrule
\end{tabular}
\end{table}
 
\paragraph{Observed behaviour} We computed the fraction of observations satisfying $\bigl|\|x\|_p / (n\mu_p)^{1/p} - 1\bigr| < \delta$ as a function of~$p$ for the original and zero-imputed datasets. Figure \ref{fig:gene_expr} shows results for $P_{\mathrm{gap}} \in \{0.01, 0.02, 0.03, 0.04\}$. The Wasserstein distances and KS $p$-values between the original and imputed distributions are: $P_{\mathrm{gap}} = 0.01$: $W = 13.88$, $p = 0.988$ (150 effective dimensions out of 154); $P_{\mathrm{gap}} = 0.02$: $W = 19.74$, $p = 0.670$; $P_{\mathrm{gap}} = 0.03$: $W = 24.19$, $p = 0.245$; $P_{\mathrm{gap}} = 0.04$: $W = 27.89$, $p = 0.088$. Even at $P_{\mathrm{gap}} = 0.01$, where the distributions remain statistically indistinguishable by the KS test, the concentration behaviour for small~$p$ diverges between the original and imputed data (Figure \ref{fig:gene_expr}).
 
\begin{figure}[htbp]
\centering
\begin{minipage}[t]{0.48\textwidth}\centering
\includegraphics[width=\textwidth]{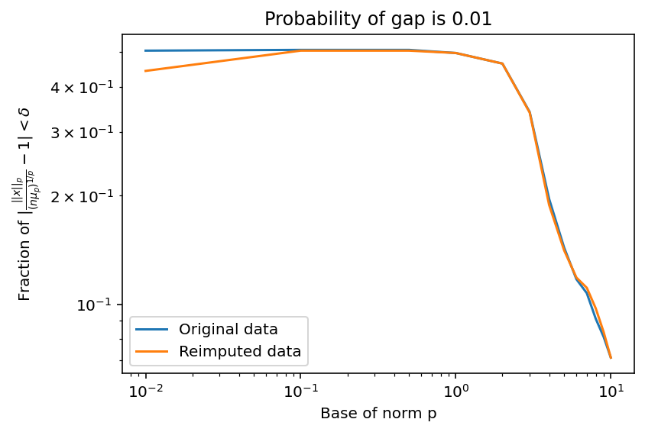}
\end{minipage}\hfill
\begin{minipage}[t]{0.48\textwidth}\centering
\includegraphics[width=\textwidth]{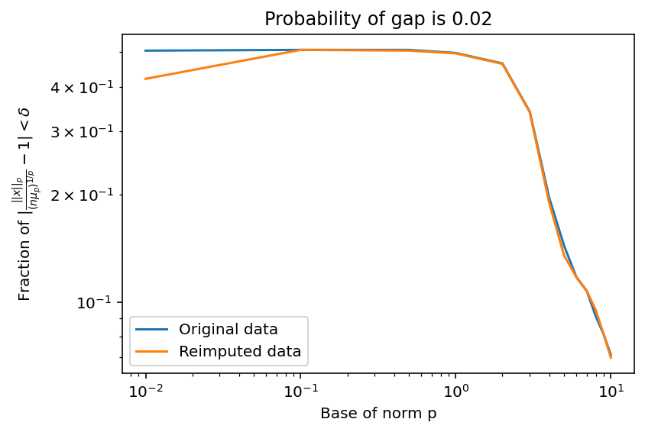}
\end{minipage}\\[4pt]
\begin{minipage}[t]{0.48\textwidth}\centering
\includegraphics[width=\textwidth]{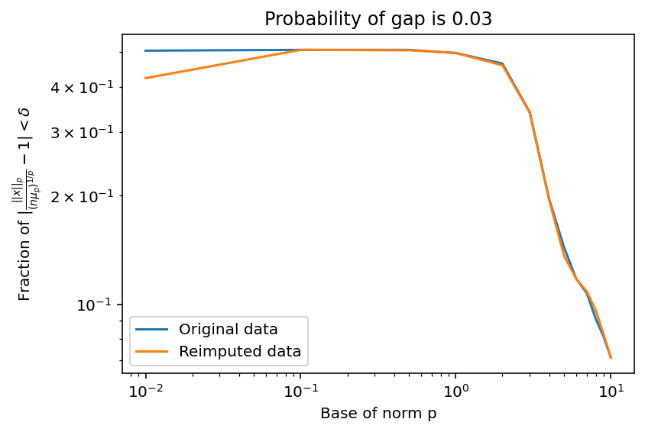}
\end{minipage}\hfill
\begin{minipage}[t]{0.48\textwidth}\centering
\includegraphics[width=\textwidth]{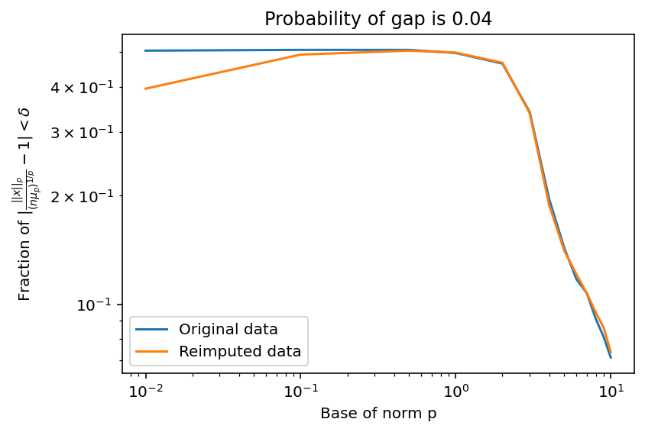}
\end{minipage}
\caption{Gene expression data: fraction of observations with $\bigl|\|x\|_p/(n\mu_p)^{1/p} - 1\bigr| < \delta$ as a function of~$p$, for the original data and after zero imputation with $P_{\mathrm{gap}} \in \{0.01, 0.02, 0.03, 0.04\}$.}
\label{fig:gene_expr}
\end{figure}

\subsubsection{Mode-shifting in gene expression data}\label{sec:gene_mode}
 
As a further illustration of how minor preprocessing choices can shift the concentration regime, we compared two versions of the gene expression dataset. The first is the standardised dataset with constant attributes removed. The second is obtained from the first by re-centring all attributes with fewer than 20 unique values in such a way that the modes of the re-centred attributes are all zero; this affected 398 attributes and introduced 315{,}428 zero entries into the data matrix. The Wasserstein distance between the two datasets is $W = 1.791$, and the KS test rejects the null hypothesis of equal distributions ($p < 0.001$). Figure \ref{fig:gene_mode} shows the concentration fraction curves for the two datasets. The mode-shifted version exhibits reduced concentration for small~$p$, consistent with Theorem \ref{thm:concentration_breaks:general}: the introduction of atoms at zero breaks concentration for small~$p$ even though the overall data distribution changed only modestly.
 
\begin{figure}[htbp]
\centering
\includegraphics[width=0.6\textwidth]{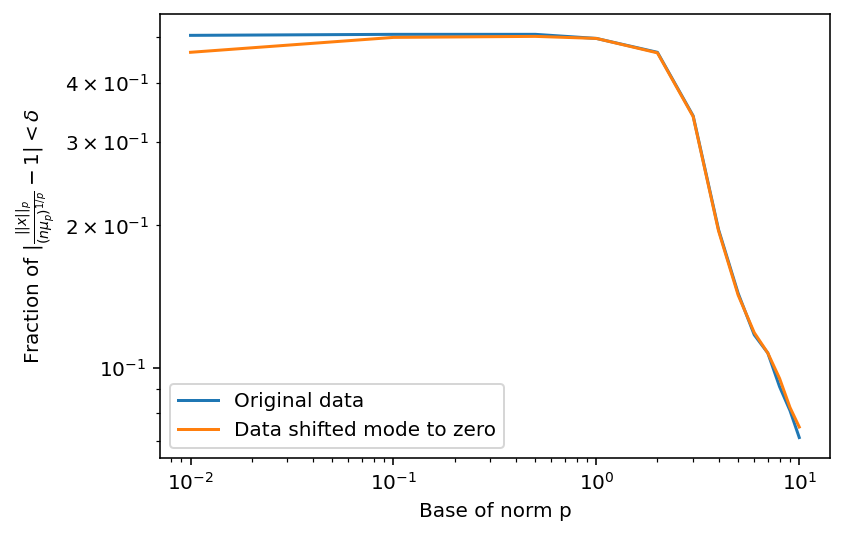}
\caption{Gene expression data: concentration fraction for the standardised dataset (blue) and the mode-shifted dataset (orange), in which 398 low-unique-value attributes were shifted so that their shifted mode equals zero.}
\label{fig:gene_mode}
\end{figure}
 
\subsubsection{Summary}
These experiments confirm the prediction of Corollary~\ref{cor:anti}: introducing even a small probability of zero atoms---well within the Wasserstein neighbourhood of the original distribution---suffices to break concentration for small~$p$. The effect is visible across synthetic data, a clinical dataset (WDBC), and gene expression data, and it is robust to variations in dimension ($n = 30$ to $n = 20{,}264$) and sample size. One practically relevant implication is the potential impact of a common imputation method in datasets with missing data: mean value imputation may impact concentration for small $p$. A related  consequence is that practitioners can exploit the anti-concentration regime of Theorem~\ref{thm:concentration_breaks:general} by applying simple preprocessing operations (zero imputation, mode shifting).

\subsection{Sparse vs.\ dense embeddings in retrieval-augmented generation}
\label{sec:retrieval}

Modern information retrieval systems, including those underpinning retrieval-augmented generation (RAG) in large language models~\cite{Lewis2020RAG,Gao2024RAGSurvey}, encode documents and queries as vectors in high-dimensional spaces. Two fundamentally different classes of representations are dense and sparse embeddings, as described below, and their complementary strengths and weaknesses have a direct connection to the concentration results of this paper.

\subsubsection{Dense embeddings}
Models such as Sentence-BERT~\cite{Reimers2019SentenceBERT}, E5~\cite{Wang2022E5}, and DPR~\cite{Karpukhin2020DPR} produce fixed-length vectors $f_\theta(x) \in \Real^n$ with $n$ typically between 384 and 4096, where every coordinate is a  floating-point value obtained from continuous nonlinear transformations (transformer layers, mean-pooling, and optional $\ell^2$ normalisation). Their strengths are: \emph{Semantic generalisation}--- empirical evidence suggests that dense embeddings capture meaning beyond lexical overlap, so a query for ``electric cars'' will retrieve documents about ``zero-emission vehicles'' even if no terms match exactly~\cite{Karpukhin2020DPR,Reimers2021Curse}; 
 \emph{Compact storage}---a 384-dimensional float32 vector occupies ${\sim}1.5$\,KB, enabling efficient approximate nearest neighbour  search ~\cite{Malkov2020HNSW}.

Their weaknesses are equally characteristic: \emph{Vocabulary mismatch on rare terms}---dense models struggle with proper nouns, technical jargon, and domain-specific tokens underrepresented in training data~\cite{Formal2024Efficient,Reimers2021Curse}; 
 \emph{Lack of interpretability}---each coordinate of a dense vector has no human-readable meaning; and 
\emph{Out-of-domain fragility}---on the BEIR benchmark~\cite{Thakur2021BEIR}, which tests zero-shot transfer across 13 diverse domains, the classical sparse baseline BM25 outperforms many dense retrieval models on specialised corpora~\cite{Formal2024Efficient,Thakur2021BEIR}.

Whilst it is difficult to guarantee that dense embeddings have no atoms at zero, there is no reason to assume that such atoms will naturally occur in these systems without any regularisation enforcing sparse representations.  

\subsubsection{Sparse embeddings}
Models such as SPLADE~\cite{Formal2021SPLADE,Formal2021SPLADEv2,Lassance2024SPLADEv3}, ELSER~\cite{Elastic2023ELSER}, and uniCOIL~\cite{Lin2021COIL} produce vectors in the vocabulary space $\Real^{|\mathcal{V}|}$ with $|\mathcal{V}| \approx 30{,}522$ (the BERT WordPiece vocabulary). The vast majority of dimensions are exactly zero. Empirical measurements show that the number of nonzero (active) dimensions per passage depends on the regularisation strength: the standard \texttt{splade-cocondenser-ensembledistil} model activates approximately 56 out of 30{,}522 dimensions on average, giving a sparsity of ${\mathbb P}(x_i = 0) \approx 0.998$~\cite{SentenceTransformers2025Sparse}; the most strongly regularised SPLADE variant (FLOPS${}=0.05$) retains only ${\sim}18$ active terms per document~\cite{Formal2021SPLADE}. Less regularised variants activate up to ${\sim}120$ terms~\cite{Formal2024Efficient}. The nonzero entries are learned term importance weights produced by the masked language model (MLM) head of BERT, with a log-saturation activation $w_j = \log(1 + \mathrm{ReLU}(s_j))$ and explicit $\ell^1$ regularisation to enforce sparsity~\cite{Formal2021SPLADE}. Their strengths are: \emph{Exact term matching}---sparse vectors retain the ability to match specific keywords, names, and codes; \emph{Term expansion}---unlike BM25, learned sparse models can expand the representation to include semantically related terms~\cite{Formal2021SPLADE,Formal2024Efficient}; \emph{Interpretability}---each nonzero dimension corresponds to a specific vocabulary token; and \emph{Inverted index compatibility}---sparse vectors can be indexed using classical algorithms (WAND, Block-Max WAND~\cite{Ding2011BlockMax}).

Their weaknesses are: 
\emph{Limited semantic depth}---sparse models cannot capture the full range of paraphrase relationships; 
\emph{Higher dimensionality}---the vocabulary dimension (${\sim}30$K) is much larger than dense dimensions (${\sim}384$); and 
\emph{Dependence on tokeniser vocabulary}---sparse models are constrained to a fixed vocabulary.

\subsubsection{Hybrid retrieval: combining sparse and dense scores}
Given the complementary failure modes, modern production retrieval systems combine both. For dense embeddings,  a typical score measuring similarity of two vectors is the cosine similarity:
\begin{equation}\label{eq:sdense}
s_{\mathrm{dense}}(d, q) = \frac{(f_\theta(q), f_\theta(d))}{\|f_\theta(q)\|_2 \, \|f_\theta(d)\|_2},
\end{equation}
Cosine similarity is explicitly related to the $\ell^2$ distance between relevant vectors:
\[
\frac{1}{2}\left\|\frac{f_\theta(q)}{\|f_\theta(q)\|_2}-\frac{f_\theta(d)}{\|f_\theta(d)\|_2}\right\|_2=1-s_{\mathrm{dense}}(d, q).
\]
For sparse embeddings, a commonly used score is the dot product of the sparse term-weight vectors:
\begin{equation}\label{eq:ssparse}
s_{\mathrm{sparse}}(d, q) = \sum_{j \in \mathcal{V}} w_q(j) \cdot w_d(j),
\end{equation}
where $w_q(j), w_d(j) \geq 0$ are the learned term importance weights, and the sum is nonzero only over co-occurring terms.

The two dominant fusion mechanisms are: \emph{(a)~Weighted linear interpolation}---normalise the scores and combine as
\begin{equation}\label{eq:hybrid}
s_{\mathrm{hybrid}}(d, q) = \alpha \cdot \hat{s}_{\mathrm{dense}}(d, q) + (1 - \alpha) \cdot \hat{s}_{\mathrm{sparse}}(d, q),
\end{equation}
where $\hat{s}$ denotes the normalised score and $\alpha \in [0,1]$ is a mixing weight~\cite{Reimers2021Curse,Dey2025Hybrid}; and \emph{(b)~Reciprocal Rank Fusion (RRF)}~\cite{Cormack2009RRF}, which combines \emph{rankings}:
\begin{equation}\label{eq:rrf}
\mathrm{RRF}(d) = \sum_{i=1}^{N} \frac{1}{k + r_i(d)},
\end{equation}
where $k = 60$ is a standard smoothing constant. RRF requires no score normalisation and has become the default fusion method in systems such as Elasticsearch~\cite{Elastic2023ELSER} and many RAG pipelines~\cite{Lewis2020RAG}.

Both fusion strategies consistently outperform either retriever alone. On the BEIR benchmark (18 datasets), Kamalloo et al.~\cite{Kamalloo2024BEIR} report that SPLADE alone achieves an average Normalised Discounted Cumulative Gain at rank 10 (see \cite{Jarvelin2002nDCG} for a definition of this metric) of $0.477$, an $11\%$ relative improvement over BM25 ($0.429$), and that a dense--sparse hybrid (Contriever + SPLADE with score fusion) reaches $0.485$, a $13\%$ relative improvement over BM25. Crucially, the hybrid outperforms BM25 on 16 of 18 datasets, with particularly large gains on corpora with high vocabulary mismatch~\cite{Kamalloo2024BEIR,Formal2024Efficient}.

\subsubsection{Connection to the results of this paper}
The concentration/anti-concentration dichotomy of Theorems~\ref{thm:concentration_breaks:general} and~\ref{prop:unip} provides a geometric explanation for why hybrid retrieval works. We have assumed that dense embeddings satisfy $\mathbb{P}(f_\theta(x)_i = 0) = 0$. By Theorem~\ref{prop:unip}, their $\ell^p$ quasi-norms concentrate uniformly in~$p$, so switching to a fractional $\ell^p$ metric cannot improve discrimination. Sparse embeddings have $\mathbb{P}(x_i = 0) = a \approx 0.998$ (Assumption \ref{assume:d}), and Theorem \ref{thm:concentration_breaks:general} guarantees anti-concentration for small~$p$.

Hybrid retrieval thus combines a representation in the \emph{concentration regime} (dense) with one in the \emph{anti-concentration regime} (sparse). The dense component provides semantic generalisation, while the sparse component provides geometric discrimination. In dense space, all pairwise distances collapse to a thin shell (Theorem~\ref{prop:unip}), so the dense retriever's ranking is based on tiny, noise-sensitive fluctuations. The sparse component provides a genuinely informative distance signal that stabilises the fused ranking.

\paragraph{Sparse retrieval scores as $\ell^p$ norms of a Hadamard product}
The sparse retrieval score~\eqref{eq:ssparse} admits a direct reformulation in terms of $\ell^p$ quasi-norms that connects it to the concentration results of this paper. Define the Hadamard (elementwise) product $z \in \Real^{|\mathcal{V}|}$ by $z_j = w_q(j) \cdot w_d(j)$. Since all weights are nonnegative, the standard sparse score is exactly the $\ell^1$ norm of~$z$:
\[
s_{\mathrm{sparse}}(d,q) = \sum_{j \in \mathcal{V}} z_j = \|z\|_1.
\]
Now consider the family of $\ell^p$ aggregations $\|z\|_p^p = \sum_{j} z_j^p$ for $p > 0$, applied to the same vector~$z$. At the two endpoints: $\|z\|_1 = s_{\mathrm{sparse}}(d,q)$ recovers the standard weighted score, while $\lim_{p \to 0^+} \|z\|_p^p = |\{j : z_j > 0\}| = \|z\|_0$ counts the number of vocabulary tokens activated in \emph{both} the query and document---that is,
\[
\|z\|_0 = |\mathrm{supp}(w_q) \cap \mathrm{supp}(w_d)|,
\]
the pure lexical overlap. Fractional $\ell^p$ with $p \in (0,1)$ interpolates continuously between these: as $p$ decreases from~$1$ toward~$0$, $\|z\|_p^p$ progressively down-weights the contribution of term importance magnitudes and shifts toward support counting. 
 
For SPLADE embeddings with $\mathbb{P}(w(j) > 0) \approx 0.002$, the Hadamard product satisfies $\mathbb{P}(z_j = 0) \approx 1 - (0.002)^2 \approx 0.999996$ (assuming independence), placing $z$ deep in the anti-concentration regime of Assumption~\ref{assume:d}. Theorem~\ref{thm:concentration_breaks:general} guarantees that $\|z\|_p$ anti-concentrates for small~$p$, meaning that fractional $\ell^p$ quasi-norms of~$z$ become discriminative across different query--document pairs.

Remarkably, as our earlier experiments in Section \ref{sec:preprocessing} showed, anti-concentration effects emerge for relatively large values of $p\geq 1$ (see Figures \ref{fig:wdbc} and \ref{fig:cube30}). This may explain why hybrid retrieval works.  
 
\paragraph{Historical context: from $\ell^0$ to $\ell^1$}
The formulation above reveals that the entire history of sparse retrieval scoring can be viewed as a progression along the $p$-axis of a single parametric family applied to $z = w_q \odot w_d$:
 
\smallskip\noindent\emph{$p = 0$ (Boolean and coordination-level matching).} The earliest retrieval models~\cite{Salton1975VSM,vanRijsbergen1979IR} scored documents by pure lexical overlap: a document either matched a Boolean query (AND of query terms present) or it did not. Coordination-level matching ranked documents by the count $|\mathrm{supp}(w_q) \cap \mathrm{supp}(w_d)|$, which is exactly $\|z\|_0$. The Jaccard coefficient $J(q,d) = |T_q \cap T_d|/|T_q \cup T_d|$ is a normalised variant of the same quantity. All term importance information is discarded.
 
\smallskip\noindent\emph{$p = 1$ with heuristic weights (TF-IDF).} The vector space model~\cite{Salton1975VSM} introduced term-frequency and inverse-document-frequency weighting, moving scoring from $\|z\|_0$ to $\|z\|_1$ with heuristic weights.
 
\smallskip\noindent\emph{$p = 1$ with probabilistic weights (BM25).} The BM25 scoring function~\cite{Robertson1994BM25} retained the $\ell^1$ inner-product structure but introduced principled probabilistic term weighting with saturation and document-length normalisation.
 
\smallskip\noindent\emph{$p = 1$ with learned weights (SPLADE).} Modern learned sparse models~\cite{Formal2021SPLADE} again use the $\ell^1$ inner product but with neural-network-learned weights and term expansion, achieving state-of-the-art performance on BEIR.
 
\smallskip
What does not appear to exist in the literature is any systematic study of the intermediate regime $p \in (0,1)$ applied to $z = w_q \odot w_d$. The classical information retrieval literature jumped directly from $\|z\|_0$ (Boolean) to $\|z\|_1$ (TF-IDF/BM25) without exploring fractional aggregation. The results of this paper provide a theoretical foundation for such an exploration: Theorem~\ref{thm:concentration_breaks:general} predicts that the fractional regime $p \in (0,1)$ may offer improved discriminative power for sparse retrieval vectors, because the extreme sparsity of the Hadamard product places it deep in the anti-concentration regime where small~$p$ breaks concentration. Whether this theoretical advantage translates into improved retrieval quality in practice is an open question that we believe merits empirical investigation.

\subsection{Neighbour embeddings: UMAP and t-SNE}
\label{sec:umap}

The most widely used dimensionality reduction methods---UMAP~\cite{McInnes2018UMAP} and t-SNE~\cite{vanderMaaten2008tSNE}---construct neighbourhood graphs in high-dimensional space and must cope with distance concentration directly. As noted in recent reviews~\cite{Bohm2025LowDim}, MDS (Multidimensional Scaling)-type methods fail because high-dimensional distances concentrate, while neighbour-embedding methods succeed by focusing on local ordering. UMAP constructs edge weights
\begin{equation}\label{eq:umap}
v_{j|i} = \exp\!\left(-\frac{d(x_i, x_j) - \rho_i}{\sigma_i}\right),
\end{equation}
where $d$ denotes a user-chosen ambient-space distance metric (typically Euclidean), $\rho_i = \min_{j} d(x_i, x_j)$ is the distance to the nearest neighbour, and $\sigma_i$ is a per-point bandwidth calibrated so that $\sum_{j \in \mathrm{kNN}(i)} v_{j|i} = \log_2(k)$. An alternative approach is t-SNE, which performs an analogous calibration via a perplexity constraint on per-point Gaussian bandwidths~\cite{vanderMaaten2008tSNE}. 

\subsubsection{Connection to this paper}
Theorem \ref{prop:unip} provides a quantitative justification for why this local rescaling is necessary: for distributions satisfying Assumption \ref{assume:c} ($\mathbb{P}(x_i = 0) = 0$), the exponential concentration of $\ell^p$ norms holds uniformly in~$p$, so the raw distance matrix becomes uninformative in high dimensions regardless of the choice of~$d$. This supports why adaptive calibration of distances is necessary in both approaches. Below we highlight issues which may emerge as a result of such calibration and present a potential unified direction for future work based on results provided in this paper.
 
\paragraph{Instability of adaptive bandwidth under concentration} Under concentration, all pairwise distances $d(x_i, x_j)$ for $j \in \mathrm{kNN}(i)$ are approximately equal to some common value, and the nearest-neighbour distance $\rho_i$ is close to this value too. The differences $d(x_i, x_j) - \rho_i$ are therefore small for all~$k$ neighbours. If $\sigma_i$ were held constant, then $v_{j|i} \approx \exp(0) = 1$ for all neighbours, the graph would carry no structural information.
 
The constraint $\sum_{j} v_{j|i} = \log_2(k)$ prevents this collapse by forcing $\sigma_i$ to adapt: since $\log_2(k) < k$, the average weight $\log_2(k)/k$ must be strictly less than~$1$, so $\sigma_i$ must shrink until it is comparable to the spread of the residual differences $\{d(x_i, x_j) - \rho_i\}_{j \in \mathrm{kNN}(i)}$. The method thus effectively ``zooms in'' on whatever variation remains within the concentration shell.
 
However, under strong concentration, this adaptive $\sigma_i$ is forced to become very small, because the differences it must resolve are themselves very small. At that point, the method is amplifying the residual fluctuations within the concentration shell, and the question is whether these fluctuations carry genuine local structure (the true nearest neighbour really is closer than the $k$-th nearest neighbour) or whether they are dominated by noise (finite-sample estimation error, floating-point effects, or the stochastic component of the data). 

\paragraph{Fractional $\ell^p$ distances as an alternative for zero-inflated data}
The results of this paper suggest a different remedy for data with atoms at zero. For zero-inflated and sparse data, e.g., single-cell RNA-seq count matrices, which are precisely the use cases for UMAP \cite{heumos2023best}, Assumption \ref{assume:b} is satisfied, and Theorem \ref{thm:concentration_breaks:general} guarantees anti-concentration of $\ell^p$ quasi-norms for small~$p$. If the ambient-space distance in~\eqref{eq:umap} is replaced by $d_p(x_i, x_j) = \|x_i - x_j\|_p$ for a suitably small $p \in (0,1)$, the distances themselves would no longer concentrate.  The spread of $\{d_p(x_i, x_j)\}_{j \in \mathrm{kNN}(i)}$ would remain genuinely large,  $\sigma_i$ does not need to shrink to pathologically small values, and the resulting edge weights $v_{j|i}$ would not amplify noise. Whether this theoretical prediction translates into improved embeddings in practice is an open question that merits empirical investigation.


\subsection{Further experimental validation -- concentration in standard embeddings}
\label{sec:experiment}
We tested the predictions of Theorems \ref{thm:concentration_breaks:general} and \ref{prop:unip} on four embedding types from modern ML.
 
\subsubsection{Setup}
We generated $M = 5{,}000$ synthetic vectors matching the following properties: (a)~the zero-atom probability $\mathbb{P}(x_i = 0)\neq 0$ and agrees with the values observed in practice;  (b)~coordinate independence (product distribution), as required by the theorems. Real embeddings may violate~(b) due to  correlations and dependencies. In this experiment, however, we shall impose this assumption to validate the theory in controlled settings.
 
The classes of models (Table \ref{tab:setup}) are: \textbf{Dense} ($n{=}384$): $x \sim \mathcal{N}(0, 0.0225 I)$, then $L^2$-normalised~\cite{Reimers2019SentenceBERT}; \textbf{Sparse} ($n{=}5000$): zero w.p.\ $0.998$, nonzero $\sim \mathrm{Exp}(1.5)$, matching SPLADE sparsity~\cite{Formal2021SPLADE,SentenceTransformers2025Sparse}; \textbf{ReLU} ($n{=}384$): $\max(0, g_i)$, $g_i \sim \mathcal{N}(0,0.09)$; \textbf{Binary} ($n{=}500$): $\mathrm{Bernoulli}(0.1)$.
 
The sparse embedding sparsity of $\mathbb{P}(x_i = 0) = 0.998$ is set to match the empirically measured sparsity of the standard \texttt{splade-cocondenser-ensembledistil} model, which activates approximately 56 out of 30{,}522 vocabulary dimensions per passage (99.82\% sparsity)~\cite{SentenceTransformers2025Sparse}. More aggressively regularised SPLADE variants reach ${\sim}18$ active terms ($>99.9\%$ sparsity)~\cite{Formal2021SPLADE}. In our simulation with $n=5{,}000$, this corresponds to ${\sim}10$ nonzero entries per vector.
 
\begin{table}[htbp]
\centering\footnotesize
\caption{Experimental setup.}\label{tab:setup}
\begin{tabular}{@{}lcccl@{}}
\toprule
Type & $n$ & $\mathbb{P}(x_i{=}0)$ & Avg.\ non-zero & Theorem (Assumption)\\
\midrule
Dense  & 384   & 0.0000 & 384 & \ref{prop:unip} (Assumption~\ref{assume:c}) \\
Sparse & 5000  & 0.9980 & 10  & \ref{thm:concentration_breaks:general}  (Assumption~\ref{assume:b}) \\
ReLU   & 384   & 0.5000 & 192 & \ref{thm:concentration_breaks:general}  (Assumption~\ref{assume:b}) \\
Binary & 500   & 0.9000 & 50  & \ref{thm:concentration_breaks:general} (Assumption~\ref{assume:b}) \\
\bottomrule
\end{tabular}
\end{table}
 
For each embedding type and $p \in \{0.01, 0.05, 0.1, 0.25, 0.5, 1, 2, 4, 10\}$, we computed the concentration probability $\mathbb{P}(1-\delta \leq \|x\|_p/(n\mu_p)^{1/p} \leq 1+\delta)$ with $\delta = 0.1$ and the relative contrast $RC_p = |\|x_1\|_p - \|x_2\|_p|/\|x_1\|_p$ for 3{,}000 random pairs.
 
\subsubsection{Results}
 
\begin{table}[htbp]
\centering\footnotesize
\caption{Concentration probability ($\delta = 0.1$).}\label{tab:conc}
\begin{tabular}{@{}ccccc@{}}
\toprule
$p$ & Dense ($n{=}384$) & Sparse ($n{=}5000$) & ReLU ($n{=}384$) & Binary ($n{=}500$) \\
\midrule
0.01 & \textbf{0.980} & 0.016 & 0.030 & 0.060 \\
0.10 & \textbf{0.991} & 0.029 & 0.156 & 0.060 \\
0.50 & \textbf{1.000} & 0.111 & 0.601 & 0.283 \\
1.00 & \textbf{1.000} & 0.180 & 0.826 & 0.545 \\
2.00 & \textbf{1.000} & 0.206 & 0.926 & 0.868 \\
10.0 & 0.863 & 0.051 & 0.620 & \textbf{1.000} \\
\bottomrule
\end{tabular}
\end{table}
 
Dense embeddings (Table \ref{tab:conc}) maintain concentration $\geq 0.98$ for $p \in [0.01, 4]$ (Theorem~ \ref{prop:unip}). With empirically grounded sparsity of $\mathbb{P}(x_i = 0) = 0.998$, the sparse embeddings show strong anti-concentration  confirming that the actual sparsity of SPLADE models places them deep in the anti-concentration regime. ReLU shows intermediate behaviour; binary shows anti-concentration for small~$p$.
 
\begin{table}[htbp]
\centering\footnotesize
\caption{Median relative contrast $RC_p$ (3{,}000 pairs).}\label{tab:rc}
\begin{tabular}{@{}ccccc@{}}
\toprule
$p$ & Dense & Sparse & ReLU & Binary \\
\midrule
0.01 & 0.040 & \textbf{1.000} & \textbf{1.000} & \textbf{1.000} \\
0.10 & 0.036 & \textbf{0.999} & \textbf{0.454} & \textbf{0.897} \\
0.50 & 0.023 & \textbf{0.613} & 0.112 & 0.246 \\
1.00 & 0.012 & \textbf{0.406} & 0.069 & 0.132 \\
2.00 & 0.000 & \textbf{0.360} & 0.056 & 0.065 \\
\bottomrule
\end{tabular}
\end{table}
 
For dense embeddings (Table \ref{tab:rc}), $RC_p < 0.04$ for all~$p$. For sparse embeddings at realistic sparsity, $RC_p$ remains above $0.36$ even at $p = 2$ and reaches $1.0$ for $p \leq 0.1$---a much wider range of useful~$p$ values than at lower sparsity. This complements our theoretical results (Theorem~\ref{thm:concentration_breaks:general}) and other empirical results presented in Sections \ref{sec:cube30}, \ref{sec:wdbc}, \ref{sec:gene_expr}, and shows that the anti-concentration effects could be amplified at high sparsity levels -- that is, they may hold over  broad ranges of $p$, including for $p\geq 1$.

\end{document}